\RequirePackage{fix-cm}
\documentclass[smallextended]{svjour3}       

\smartqed 
\usepackage{natbib}
\usepackage{helvet}
\usepackage{courier}
\usepackage{url}
\usepackage{graphicx}
\usepackage{amsfonts}
\usepackage{amsmath}
\usepackage{xcolor}
\usepackage{breqn}
\usepackage{xspace}
\usepackage{multirow} 

\def\ourmethod{\textsl{QP-WL}\xspace} 
\def\quantilenpl{\textsl{QP-NPL}\xspace} 
\def\maxnpl{\textsl{MP-NPL}\xspace} 
\def\quantileclass{\textsl{QP-CLS}\xspace} 
\def\jager{\textsl{Jager et al. (2019)}\xspace} 
\def\bx{\mathbf{x}}

\newcommand{\ceil}[1]{\left\lceil #1 \right\rceil}
\newcommand{\sign}{\text{sign}}

\begin{document}

\title{Deep Distributional Sequence Embeddings Based on a Wasserstein Loss}
\subtitle{}
\author{Ahmed Abdelwahab \and Niels Landwehr}
\institute{
A. Abdelwahab \email{AAbdelwahab@atb-potsdam.de} \at
Leibniz Institute of Agricultural Engineering and Bioeconomy e.V., Potsdam, Germany\\
\and
N. Landwehr \email{Landwehr@cs.uni-potsdam.de} \at
University of Potsdam, Department of Computer Science, Potsdam, Germany\\
Leibniz Institute of Agricultural Engineering and Bioeconomy e.V., Potsdam, Germany\\
}

\date{Received: date / Accepted: date}

\maketitle

\begin{abstract}
Deep metric learning employs deep neural networks to embed instances into a metric space such that distances between instances of the same class are small 
and distances between instances from different classes are large.
In most existing deep metric learning techniques, the embedding of an instance is given by a feature vector produced by a deep neural network
and Euclidean distance or cosine similarity defines distances between these vectors.
In this paper, we study deep distributional embeddings of sequences, where the embedding of a sequence is given by the distribution of learned deep features
across the sequence. This has the advantage of capturing statistical information about the distribution of patterns within the sequence in the embedding.
When embeddings are distributions rather than vectors, measuring distances between embeddings involves comparing their respective distributions.
We propose a distance metric based on Wasserstein distances between the distributions and a corresponding loss function for metric learning, which leads to a novel end-to-end 
trainable embedding model.
We empirically observe that distributional embeddings outperform standard vector embeddings and that training with the 
proposed Wasserstein metric outperforms training with other distance functions.
\end{abstract}

\section{Introduction}
\label{sec:introduction}

Metric learning is concerned with learning a representation or \emph{embedding} in which distances between instances of the same class are small and distances between instances of different classes are large. Deep metric learning approaches, in which the learned embedding is given by a deep neural network, have achieved state-of-the-art results in many tasks, including face verification and recognition~\citep{schroff2015facenet}, fine-grained image classification~\citep{reed2016learning}, zero-shot classification~\citep{bucher2016improving}, speech-to-text problems~\citep{gibiansky2017deep}, and speaker identification~\citep{li2017deep}. 
An advantage of metric learning is that the resulting representation directly generalizes to unseen classes, so the model does not need to be retrained every time a new class is introduced. This is, for example, a typical requirement in biometric applications, where it should be possible to register new subjects without retraining a model. Biometric systems also have to handle imposters, that is, subjects who are not registered in the database, which is not straightforward in standard classification settings.

In this paper, we study deep metric learning for sequence data, with a specific focus on biometric problems. Building on earlier work on 
\emph{quantile layers}~\citep{abdelwahab2019Quantile}, we specifically study how the distribution of learned deep features across a sequence can be represented in the learned
embedding. Quantile layers are statistical aggregation layers that characterize the distribution of 
patterns within a sequence by approximating the quantile
function of the activations of the learned filters across the sequence. Characterizing this distribution has been shown to be advantageous for biometric identification based 
on eye movement patterns~\citep{abdelwahab2019Quantile}. 
The main contribution of this paper is to develop a deep metric learning approach for distributional embeddings based on quantile layers.
Quantile layers return an estimate of the distribution of values for each learned filter across the sequence. 
Instead of a fixed-length vector representation of an instance, in our approach, the embedding of an instance is given by these sets of distributions.
When embeddings are distributions rather than simple vectors, measuring distances between the embeddings involves comparing their respective distributions. 
We propose a distance metric in the embedding space that is based on Wasserstein distances between the respective distributions. Compared to other distance functions such 
as Kulback-Leibler or Jensen-Shannon divergence, the advantage of using Wasserstein distance is that it takes into account the metric on the space in which the random variable of interest is defined. 
In our case, this means that distributions in which similar magnitudes of filter activations receive similar amounts of probability
mass will be considered close. 
We show how such embeddings can be trained end-to-end on labeled training data using metric learning techniques.

Empirically, we study the proposed approach in biometric identification problems involving eye movement, accelerometer, and EEG data. Empirical results show that  
the proposed distributional sequence embeddings outperform standard vector embeddings and that training with the 
Wasserstein metric outperforms training with other distance functions.

\sloppy
The rest of the paper is organized as follows. Section~\ref{sec:related_work} discusses related work. 
In Section~\ref{sec:quantile_layers} we review quantile layers and develop a distributional embedding architecture based on these layers. 
Section~\ref{sec:distance_function} introduces a Wasserstein-based distance metric for the proposed embedding model and from this derives a novel loss function for metric learning.
In Section~\ref{sec:empirical_study} we empirically study the proposed method and baselines.
\fussy

\section{Related work}
\label{sec:related_work}

Our work is motivated by the goal of capturing information about the distribution of patterns within a sequence in its embedding,
where the patterns are defined in terms of learned features of a deep neural network. 
It is related to other work in deep learning that aims to capture distributions of learned features using statistical aggregation layers.
\cite{wang2016learnable} proposed end-to-end learnable \emph{histogram layers} that approximate the distribution of learned features by a histogram.
Their work uses linear approximations to smoothen the sharp edges in a traditional histogram function and enable gradient flow. 
\cite{sedighi2017histogram} proposed a similar histogram-based aggregation layer, but use Gaussian kernels as a soft, differentiable 
approximation to histogram bins. \cite{abdelwahab2019Quantile} introduced \emph{quantile layers} to capture the distribution of learned features
based on an approximation of the quantile function, and empirically showed that this outperforms aggregation using histograms.
The contribution of our paper is to exploit quantile layers in metric learning, by defining distributional embeddings based on 
approximations of quantile functions and deriving loss functions for metric learning based on comparing the resulting distributions.

There is a large body of work on deep metric learning that studies different network architectures and loss functions.
For example, \cite{hadsell2006dimensionality} introduced a loss for a siamese network architecture that is based on all possible pairs
of instances in the training data, and its objective is to minimize distances between positive pairs (same class) while maximizing the distances between negative pairs (different classes). 
More recently, \cite{schroff2015facenet} introduced the triplet loss,
with links positive and negative pairs by an anchor instance. This idea has later been extended by \cite{oh2016deep} and \cite{sohn2016improved} by providing several
negative pairs linked to one positive pair to the loss function. The loss function introduced by \cite{sohn2016improved} has shown superior performance
in several studies~\citep{sohn2016improved,wu2017sampling,yuan2017hard}.
Our method builds on these established deep metric learning techniques,
but extends them by replacing vector embeddings with distributional embeddings, which
requires corresponding changes in distance calculations and the loss function.

Distributional embeddings have also been studied in natural language processing in the context of word embeddings.
Traditional word embedding models such as \emph{word2vec} represent words as vectors in a metric space such 
that semantically similar words are mapped to similar vectors~\citep{mikolov2013distributed}. 
\cite{vilnis2014word} extend this idea by mapping each word to a Gaussian distribution (with diagonal covariance), which naturally characterizes uncertainty
about the embedding. \cite{athiwaratkun2017multimodal} further extend this model by replacing the Gaussian distribution with a mixture of Gaussians, where the multimodal
mixture can capture multiple meanings of the same word. The motivation for these distributional embeddings is somewhat different from our motivation in this paper:
while the distribution in our model results from the inner structure of the instance being mapped (distribution of patterns within a sequence), the distribution 
in the model by~\cite{vilnis2014word} captures remaining uncertainty and is inferred during training. Another difference in the work by~\cite{vilnis2014word}
is that their model is trained in an unsupervised fashion, while we study supervised metric learning.
An approach similar to that of~\cite{vilnis2014word} has also been taken by~\cite{bojchevski2017deep} in order to map nodes of an attributed graph onto Gaussian distributions
that function as an embedding representation. This is again an unsupervised approach, and specific to the task of node embedding.

More generally, deep metric learning models have been recently used in different application domains featuring sequential data, including
natural language processing~\citep{mueller2016siamese, neculoiu2016learning}, computer vision~\citep{mclaughlin2016recurrent, wu2018and} 
and speaker identification~\citep{li2017deep,chung2018voxceleb2}, but these approaches are based on vector embeddings rather than distributional embeddings.

\section{Quantile Layers and Distributional Sequence Embeddings}
\label{sec:quantile_layers}

This section reviews \emph{quantile layers} as introduced by~\cite{abdelwahab2019Quantile} and discusses how they can be used to define distributional 
embeddings of variable-length sequences.

In this paper, we focus on variable-length sequences and deep convolutional neural network architectures that produce embeddings of such sequences. 
Typically, network architectures for such sequences would employ stacked convolution layers to extract informative features from the sequence, and in the last layer use some form of 
global pooling to transform the remaining variable-length representation into a fixed-length vector representation. 
Global pooling achieves this transformation by performing a simple aggregate operation such as taking the maximum or average over the filter activations across the sequence.
This has the potential disadvantage that most information about the distribution of the filter activations is lost, which might be informative for the task at hand.
In contrast, quantile layers aim to preserve as much information as possible about the distribution of filter activations along the sequence by approximating the quantile function
of this distribution. Earlier work has shown that this information can be informative for sequence classification, substantially increasing predictive 
accuracy~\citep{abdelwahab2019Quantile}.

In this paper, we use quantile layers for defining distributional embeddings of sequences. 
We assume that instances are given by variable-length sequences of the form 
$\mathbf{s} = (\bx_1,...,\bx_T)$ where $\bx_t \in \mathbb{R}^D$ 
is a vector of attributes that describes the sequence element at position $t$. We denote the space of all such sequences with $D$ attributes
by $\mathcal{S}_D = \bigcup_{T=1}^{\infty} \mathbb{R}^{T \times D}$. When a sequence is processed by a convolutional deep neural network architecture $\Gamma$, 
which we take to be the network without any final global aggregation layers, the result is a variable-length representation of the 
instance over $K$ filters. We denote this mapping by $\Gamma:\mathcal{S}_D \rightarrow \mathcal{S}_K$. 
Details of the deep convolutional architectures we employ are given in Section~\ref{sec:empirical_study}.
For $\mathbf{s} \in \mathcal{S}_D$ and $k \in \{1,...,K\}$ we will use $\Gamma_k(\mathbf{s})$ to denote the variable-length sequence of 
activations of filter $k$ produced by the network for sequence $\mathbf{s}$.

As in~\cite{abdelwahab2019Quantile} we use quantile functions in order to characterize the distribution of filter activations across the sequence $\Gamma_k(\mathbf{s})$.
Let $x \in \mathbb{R}$ be a real-valued random variable, let $p(x)$ denote its density and $F(x)$ its cumulative distribution function. The quantile function for $x$ is defined by
\begin{equation*}
Q(r) = \inf\{x \in \mathbb{R}: F(x) \geq r \}
\end{equation*}
where $\inf$ denotes the infimum. 
If $F$ is continuous and strictly monotonically increasing, $Q$ is simply the inverse of $F$.

Let $\mathcal{X} = \{x_1,...,x_{N}\}$ be a sample of the random variable $x$, that is, 
$x_n \sim p(x)$ for $n \in \{1,...,N\}$. 
The empirical quantile function \mbox{$\hat{Q}_{\mathcal{X}}: (0,1] \rightarrow \mathbb{R}$} is a non-parametric estimator of the quantile function $Q$. It is defined by 
\begin{equation}
\label{eq:empirical_quantile_function}
\hat{Q}_{\mathcal{X}}(r) = \inf\{x \in \mathbb{R}: r \leq \hat{F}_{\mathcal{X}}(x)\}
\end{equation}
where
$\hat{F}_{\mathcal{X}}(x) = \frac{1}{N}\sum_{i=1}^N I(x_i \leq x)$ is the empirical cumulative distribution function and $I(x_i \leq x) \in \{0,1\}$ is an indicator. 
$\hat{Q}_{\mathcal{X}}(r)$ is a piecewise constant function that is essentially obtained by sorting the samples in $\mathcal{X}$.
More formally, let $\pi$ be a permutation that sorts the $x_i$, that 
is, $x_{\pi(i)} \leq x_{\pi(i+1)}$ for $1 \leq i \leq N-1$. Then $\hat{Q}_{\mathcal{X}}(r) = x_{\pi(\ceil{r N})}$, where $\ceil{x}$ denotes the smallest integer larger or equal to $x$.  
The empirical quantile function $\hat{Q}_{\mathcal{X}}$ faithfully approximates the quantile function $Q$ in the sense that 
$|\hat{Q}_{\mathcal{X}}(r) - Q(r)|$ converges almost surely to zero if $N \rightarrow \infty$ and $Q$ is continuous at $r$~\citep{resnick2013extreme}.

To enable gradient flow in end-to-end learning, we will work with a piecewise linear interpolation of the piecewise constant function $\hat{Q}_{\mathcal{X}}(r)$. 
For $i \in \{1,...,N\}$ 
and $r \in [\frac{n-1}{N},\frac{n}{N}]$ we can define a linear approximation by 
\begin{equation*}
\tilde{Q}_{\mathcal{X}}(r) = N (x_{\pi(n+1)}-x_{\pi(n)})r + n x_{\pi(n)} + (1-n)x_{\pi(n+1)} \text{ \ \ } \left(r \in \left[\frac{n-1}{N},\frac{n}{N}\right]\right)
\end{equation*}
where we define $x_{\pi(N+1)} = x_{\pi(N)}$ to handle the right interval border.
Combining the linear approximations over the different $n$, we obtain for $r \in [0,1]$ the piecewise linear approximation
\begin{equation*}
\tilde{Q}_{\mathcal{X}}(r) = \sum_{n=1}^{N} \tilde{\delta} (r,n) \left( N (x_{\pi(n+1)}-x_{\pi(n)})r + n x_{\pi(n)} + (1-n)x_{\pi(n+1)}   \right)
\end{equation*}
where $\tilde{\delta}(r,n)$ is an indicator function that is defined as one if $r \in [\frac{n-1}{N},\frac{n}{N}]$ and zero otherwise.
The piecewise linear approximation $\tilde{Q}_{\mathcal{X}}(r)$ of the quantile function depends on the sample size $N$, because there are $N$ linear segments.
To arrive at an approximation of the quantile function that is independent of the number of samples, we define a further piecewise linear approximation of $\tilde{Q}_{\mathcal{X}}(r)$ 
using $M$ sampling points $\sigma(\alpha_1),...,\sigma(\alpha_{M})$, where $\sigma(\alpha) = (1+exp(-\alpha))^{-1}$ is the sigmoid function and $\alpha_i \in \mathbb{R}$ are
parameters with $\alpha_i \leq \alpha_{i+1}$. 
Formally, we define
\begin{equation}
\label{eq:qbar}
\bar{Q}_{\mathcal{X}}(r) = \sum_{i=0}^{M} \bar{\delta}(r,i) (a_{\mathcal{X},i} r+b_{\mathcal{X},i})
\end{equation}
where
\begin{align}
&a_{\mathcal{X},i} = \frac{\tilde{Q}_{\mathcal{X}}(\sigma(\alpha_{i+1})) - \tilde{Q}_{\mathcal{X}}(\sigma(\alpha_{i}))}{\sigma(\alpha_{i+1})-\sigma(\alpha_{i})} \label{eq:a_Xm}\\
&b_{\mathcal{X},i} = \tilde{Q}_{\mathcal{X}}(\sigma(\alpha_{i}))  - \sigma(\alpha_i) \frac{\tilde{Q}_{\mathcal{X}}(\sigma(\alpha_{i+1})) - \tilde{Q}_{\mathcal{X}}(\sigma(\alpha_{i}))}{\sigma(\alpha_{i+1})-\sigma(\alpha_{i})}, \label{eq:b_Xm}
\end{align}
 $\bar{\delta}(r,i)$ is an indicator function that is one if $r \in [\sigma(\alpha_i),\sigma(\alpha_{i+1})]$ and zero otherwise, 
and we have introduced $\alpha_{0}=-\infty$ and $\alpha_{M+1}=\infty$ to handle border cases. The function $\bar{Q}_{\mathcal{X}}(r)$ provides a piecewise linear
approximation of the quantile function using $M+1$ line segments, independently of the sample size $N$. The parameters $\alpha_i$ are learnable model parameters in the
deep neural network architectures that we study in Section~\ref{sec:empirical_study}.

We are now ready to define the distributional embedding for an instance, which is obtained by passing the instance through the neural network $\Gamma$ and
for each filter in the output of $\Gamma$ approximating the quantile function of the filter activations by the piecewise linear function $\bar{Q}$.
\begin{definition}[Distributional embedding of sequence] 
\label{def:embedding}
Let $\mathbf{s} \in \mathcal{S}_D$ and let $\Gamma$ denote a convolutional neural network structure. The distributional embedding of sequence $\mathbf{s}$ is given by the vector of piecewise linear functions
\begin{equation}
\label{eq:embedding}
\Psi_{\Gamma}(\mathbf{s}) = \left(\bar{Q}_{\Gamma_1(\mathbf{s})},...,\bar{Q}_{\Gamma_K(\mathbf{s})}\right)
\end{equation}
where $\bar{Q}_{\Gamma_k(\mathbf{s})}$ is defined by Equation~\ref{eq:qbar} using $\mathcal{X} = \Gamma_k(\mathbf{s})$. Here, we slightly generalize the notation by 
identifying the sequence of observations $\Gamma_k(\mathbf{s})$ with the corresponding set of observations.
\end{definition}

We note that due to the piecewise linear approximations, gradients can flow through the entire embedding architecture, both to parameters $\alpha_m$ and the 
weights in the deep neural network structure $\Gamma$. This includes the sorting operation, where gradients can be passed through by reordering the gradient backpropagated from the layer above according to the sorting indices $\pi$.

\section {A Wasserstein Loss for Distributional Embeddings}
\label{sec:distance_function} 

For training the embedding model, we will use deep metric learning approaches which train model
parameters such that instances of the same class are close and instances of different classes are far apart in the embedding space. In order to apply such approaches, 
a distance metric needs to be defined on the embedding space. 

\subsection{Distances Between Distributional Embeddings}
As discussed in Section~\ref{sec:quantile_layers}, in our setting embeddings of instances are given by distributions. 
Measuring the distance between two embeddings thus means comparing their respective distributions.
Different approaches to measure distances between probability distributions have been discussed in the literature. One of the most widely used distance functions
between distributions is the Kullback-Leibler divergence. However, this measure is asymmetric and can result in infinite distances, and is therefore not a metric.
A metric based on the Kullback-Leibler divergence is the square root of the Jensen-Shannon divergence, which is symmetric, bounded between zero and $\sqrt{log(2)}$, and satisfies the triangle inequality. 
However, this metric does not yield useful gradients in case the distributions being compared have disjoint support, which in our case would occur if two sequences with non-overlapping ranges of filter values are 
compared. To illustrate, let $q_1$ and $q_2$ denote densities with disjoint support $A_1$ and $A_2$, and let $m(x)=\frac{q_1(x)+q_2(x)}{2}$.
Then the Jensen-Shannon divergence $J$ of $q_1$ and $q_2$ is  
\begin{align*}
&J(q_1,q_2)= \frac{1}{2}\int_{A_1 \cup A_2}{q_1(x) log\left(\frac{q_1(x)}{m(x)}\right)}d x+\frac{1}{2}\int_{A_1 \cup A_2}{q_2(x) log\left(\frac{q_2(x)}{m(x)}\right)}d x\\
&\hspace{18mm}= \frac{1}{2}\int_{A_1}{q_1(x) log\left(2\frac{q_1(x)}{q_1(x)}\right)}d x+\frac{1}{2}\int_{A_2}{q_2(x) log\left(2\frac{q_2(x)}{q_2(x)}\right)}d x\\
&\hspace{18mm}=log(2)
\end{align*}
independently of the distance between $A_1$ and $A_2$, resulting in a gradient of zero. 

A different class of distance functions which are increasingly being studied in machine learning~\citep{frogner2015learning, gao2016distributionally,arjovsky2017wasserstein} 
are Wasserstein distances. 
Wasserstein distances are based on the idea of optimal transport plans. 
They do not suffer from the zero-gradient problem exhibited by the Jensen-Shannon divergence, 
because they take into account the metric of the underlying space. They also guarantee continuity under mild assumptions, which is not the case for the Jensen-Shannon divergence as illustrated by \cite{arjovsky2017wasserstein}.
In the general case, the $p$-Wasserstein distance (for $p \in \mathbb{N}$) between two probability measures $\rho_1$ and $\rho_2$ over a space $\mathcal{M}$ with metric $d$ can be defined as
\begin{equation}
W_p(\rho_1,\rho_2) = \left( \inf_{\pi\in\mathcal{J}(\rho_1,\rho_2)} \int_{\mathcal{M}\times\mathcal{M}}{d(x,y)^p d\pi(x,y)}\right)^{\frac{1}{p}}
\end{equation}
where $\mathcal{J}(\rho_1,\rho_2)$ denotes the set of all joint measures on $\mathcal{M} \times\mathcal{M}$ with marginals $\rho_1$ and $\rho_2$. 
For the purpose of this paper, we are interested in the case of real-valued random variables.
If $q_1(x_1)$ and $q_2(x_2)$ are two densities defining distributions over real-valued random variables, $x_i \in \mathbb{R}$,
the $p$-Wasserstein distance between $q_1$ and $q_2$ is given by
\begin{equation}
\label{eq:wasserstein_1d}
W_p(q_1,q_2) = \left( \inf_{q \in \mathcal{J}(q_1,q_2)} \iint{|x_1-x_2|^p q(x_1,x_2)d x_1 d x_2} \right)^{\frac{1}{p}}
\end{equation}
where $\mathcal{J}(q_1,q_2)$ defines the set of all joint distributions over $x_1$, $x_2$ which have marginals $q_1$ and $q_2$. A joint distribution $q \in \mathcal{J}(q_1,q_2)$
can be seen as a \emph{transport plan}, that is, a way of moving probability mass from density $q_1$ such that the resulting density is $q_2$, in the sense that
$q(x_1,x_2)$ indicates how much mass is moved from $q_1(x_1)$ to $q_2(x_2)$. The quantity $\iint{|x_1-x_2|^p q(x_1,x_2)d x_1 d x_2}$ is the cost of the transport plan,
which depends on the amount of probability mass moved, $q(x_1,x_2)$, and the distance by which the mass has been moved, $|x_1-x_2|^p$.
The infimum over the set $\mathcal{J}(q_1,q_2)$ means that the distance between the distributions is given by the optimal transport plan, which intuitively characterizes
the minimum changes that need to be made to $q_1$ in order to transform it into $q_2$. For $p=1$ the distance is therefore also called the \emph{Earth Mover Distance}.
The advantage of this measure is that it takes into account the metric in the underlying space, as can be seen from Figure~\ref{fig:dist}.  
Here, $q_1$ is closer to $q_2$ than it is to $q_3$ in the sense that the probability mass needs to be moved less far. Thus, $W_p(q_1,q_2) < W_p(q_1,q_3)$, while the Jensen-Shannon
distances between the two pairs of distributions would be identical.

\begin{figure}[t]
\centering
\includegraphics[width=0.8\linewidth]{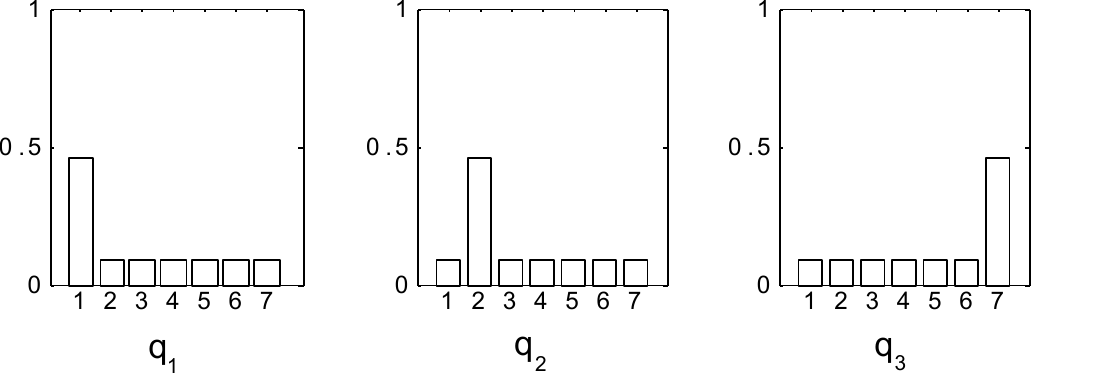}
\caption{
According to the Wasserstein metric, distributions $q_1$ and $q_2$ are closer than $q_1$ and $q_3$, while distances would be identical under the Jensen-Shannon measure.
}
\label{fig:dist}
\end{figure}

\sloppy

Because Wasserstein distances are defined in terms of optimal transport plans, computing them in general requires solving non-trivial optimization problems. 
However, for the case of real-valued random variables $x_i \in \mathbb{R}$, there is a simple closed-form solution to the infimum in Equation~\ref{eq:wasserstein_1d}.
Let $x_1 \sim q_1$, $x_2 \sim q_2$ with $x_i \in \mathbb{R}$.
According to \cite{cambanis1976inequalities}, the function $K(x_1,x_2)=|x_1-x_2|^p$ for $p \geq 1$ is quasi-antitone and therefore the infimum of the expectation of this function over the set of all joint distributions, $\inf_{q \in \mathcal{J}(q_1,q_2)} E[K(x_1,x_2)]$, is given by $\int_0^1 K(Q_1(r), Q_2(r)) d r$, where
$Q_i(r)=\inf\{t:q_i(x_i \leq t) \geq r\}$ is the quantile function to the density $q_i$.
We can thus rewrite Equation~\ref{eq:wasserstein_1d} as
\begin{equation}
W_p(q_1,q_2) = \left(\int_0^1{|Q_1(r)-Q_2(r)|^p} d r \right)^{\frac{1}{p}}.
\end{equation}
\fussy 

We now define the distance between two embeddings $\Psi_{\Gamma}(\mathbf{s})$ and $\Psi_{\Gamma}(\mathbf{s}')$ as the Wasserstein distance between the 
approximate representation of the quantile functions in the embedding \textit{}as defined by Definition~\ref{def:embedding},
summed over the different filters $k$.
\begin{definition}
\label{def:distance}
Let $\mathbf{s}, \mathbf{s}' \in \mathcal{S}_D$, let $\Gamma$ denote a convolutional neural network architecture, and let $\Psi_{\Gamma}(\mathbf{s})$ and $\Psi_{\Gamma}(\mathbf{s}')$ denote 
the distributional embeddings of $\mathbf{s}$, $\mathbf{s}'$ as defined by Definition~\ref{def:embedding}.
Then we define the distance between the embeddings as
\begin{equation}
\label{eq:distance}
d_p(\Psi_{\Gamma}(\mathbf{s}),\Psi_{\Gamma}(\mathbf{s}')) = \sum_{k=1}^K \left(\int_0^1{| \bar{Q}_{\Gamma_k(\mathbf{s})}(r) - \bar{Q}_{\Gamma_k(\mathbf{s}')}(r)|^p} d r \right)^{\frac{1}{p}}
\end{equation}
\end{definition}
The next proposition gives a closed-form result for computing $d_p(\Psi_{\Gamma}(\mathbf{s}),\Psi_{\Gamma}(\mathbf{s}'))$.
\vspace*{-0.5cm}
\begin{proposition}
\label{prop:closed_form}
Let $\mathbf{s}, \mathbf{s}' \in \mathcal{S}_D$, let $\Gamma$ denote a convolutional neural network architecture, let $\Psi_{\Gamma}(\mathbf{s})$ and $\Psi_{\Gamma}(\mathbf{s}')$ denote 
the distributional embeddings of $\mathbf{s}$, $\mathbf{s}'$, and let $d_p(\Psi_{\Gamma}(\mathbf{s}),\Psi_{\Gamma}(\mathbf{s}'))$ denote their distance as defined by Definition~\ref{def:distance}. 
Then
\begin{align}
&d_p(\Psi_{\Gamma}(\mathbf{s}),\Psi_{\Gamma}(\mathbf{s}')) =
\sum_{k=1}^K \biggl( \sum_{i=0}^M  \frac{(\bar{a}_{i,k}\sigma(\alpha_{i+1})+\bar{b}_{i,k})|\bar{b}_{i,k}\sigma(\alpha_{i+1})+\bar{b}_{i,k}|^p}{\bar{a}_{i,k}(p+1)}\notag\\
&\hspace{45mm}-\frac{(\bar{a}_{i,k}\sigma(\alpha_i)+\bar{b}_{i,k})|\bar{a}_{i,k} \sigma(\alpha_i)+\bar{b}_{i,k}|^p}{\bar{a}_{i,k}(p+1)} \biggr)^{\frac{1}{p}}\label{eq:distance_closed_form}
\end{align}
with 
\begin{align*}
&\bar{a}_{i,k} = a_{\Gamma_k(\mathbf{s}),i} - a_{\Gamma_k(\mathbf{s}'),i}\\ 
&\bar{b}_{i,k} = b_{\Gamma_k(\mathbf{s}),i} - b_{\Gamma_k(\mathbf{s}'),i} 
\end{align*} 
where $a_{\mathcal{X},i}$ and $b_{\mathcal{X},i}$ for $\mathcal{X} \in \{\Gamma_k(\mathbf{s}),\Gamma_k(\mathbf{s}')\}$ are defined by Equations~\ref{eq:a_Xm} and~\ref{eq:b_Xm}, 
$\sigma$ is the sigmoid function, and as above we have introduced $\alpha_{0}=-\infty$ and $\alpha_{M+1}=\infty$ to handle border cases.
\end{proposition}
\begin{proof}[Proposition~\ref{prop:closed_form}]
Starting from Definition~\ref{def:distance} and plugging in $\bar{Q}_{\Gamma_k(\mathbf{s})}$ as defined by Equation~\ref{eq:qbar}, we see that
\begin{align}
& \int_0^1{| \bar{Q}_{\Gamma_k(\mathbf{s})}(r) - \bar{Q}_{\Gamma_k(\mathbf{s}')}(r)|^p} d r \notag \\
& \hspace{20mm}= \int_0^1{| \sum_{i=0}^{M} \bar{\delta}(r,i) \left((a_{\Gamma_k(\mathbf{s}),i}-a_{\Gamma_k(\mathbf{s}'),i}) r + b_{\Gamma_k(\mathbf{s}),i} - b_{\Gamma_k(\mathbf{s}'),i} \right) |^p} d r \notag \\
& \hspace{20mm}= \sum_{i=0}^{M} \int_{\sigma(\alpha_i)}^{\sigma(\alpha_{i+1})}{ | \bar{a}_{i,k} r + \bar{b}_{i,k} |^p} d r \label{eq:limits}\\
& \hspace{20mm}= \sum_{i=0}^{M} \frac{(\bar{a}_{i,k} r + \bar{b}_{i,k})|\bar{a}_{i,k} r +\bar{b}_{i,k}|^p}{\bar{a}_{i,k} (p+1)} \Bigg|_{\sigma(\alpha_i)}^{\sigma(\alpha_{i+1})} \label{eq:antiderivative}
\end{align} 
\end{proof}
where in Equation~\ref{eq:antiderivative} we use the notation $G(r)\vert_a^b = G(b)-G(a)$.
In Equation~\ref{eq:limits} we integrate over subintervals $[\sigma(\alpha_i),\sigma(\alpha_{i+1})]$ of the interval $[0,1]$, and can therefore remove the indicator function $\bar{\delta}(r,i)$. 
In Equation~\ref{eq:antiderivative} we solve the integral, where we exploit that according to product and chain rules 
\begin{align*}
&\frac{\partial}{\partial r}\frac{(\bar{a}_{i,k} r + \bar{b}_{i,k})|\bar{a}_{i,k} r +\bar{b}_{i,k}|^p}{\bar{a}_{i,k} (p+1)}  \\
&\hspace{5mm} =\frac{\bar{a}_{i,k}|\bar{a}_{i,k} r + \bar{b}_{i,k} |^p +(\bar{a}_{i,k} r + \bar{b}_{i,k}) p |\bar{a}_{i,k} r + \bar{b}_{i,k} |^{p-1} \sign(\bar{a}_{i,k} r + \bar{b}_{i,k}) \bar{a}_{i,k}}{\bar{a}_{i,k}(p+1)}\\
&\hspace{5mm}=|\bar{a}_{i,k} r + \bar{b}_{i,k} |^p.
\end{align*}
The claim directly follows from Equation~\ref{eq:antiderivative}. \qed

An important note with respect to the distance function $d_p(\Psi_{\Gamma}(\mathbf{s}),\Psi_{\Gamma}(\mathbf{s}'))$ is that its
 closed-form computation given by Proposition~\ref{prop:closed_form} allows gradients to be propagated through distance computations (as well as through embedding computations as discussed in Section~\ref{sec:quantile_layers}) to the parameters of the model $\Gamma$ defining the embedding. Moreover, all computations can be expressed using standard building blocks available in common deep learning frameworks, such that all gradients are available through automatic differentiation.

\subsection{Loss Function}

Deep metric learning trains models with loss functions that drive the model towards minimizing distances between pairs of instances from the 
same class (positive pairs) while maximizing distances between pairs of instances from different classes (negative pairs). 
Existing approaches differ in the way negative and positive pairs are selected and the exact formulation of the loss. 
For example, triplet-based losses as introduced by~\cite{schroff2015facenet} compare the distance between an anchor instance 
and another instance from the same class (positive pair) to the distance between the anchor instance and an instance 
from a different class (negative pair).
However, comparing a positive pair with only a single negative pair does not take into account the distance to other classes and can thereby 
lead to suboptimal gradients; more recent approaches therefore often consider several negative pairs for 
each positive pair~\citep{oh2016deep,sohn2016improved}. Inspired by these approaches, we consider several negative pairs for each positive pair,
leading to a loss function of the form
\begin{equation*}
\mathcal{L} = \sum_{(\mathbf{s_1},\mathbf{s_2}) \in \mathcal{P}}\sum_{\substack{(\mathbf{s}_3,\mathbf{s}_4) \in \mathcal{N} \\ \mathbf{s}_3\in\{\mathbf{s}_1,\mathbf{s}_2\}}} \ell(\mathbf{s}_1,\mathbf{s}_2,\mathbf{s}_3,\mathbf{s}_4)
\end{equation*}
where $\mathcal{P} \subset \mathcal{S}_D \times \mathcal{S}_D$ is a set of positive pairs and $\mathcal{N} \subset \mathcal{S}_D \times \mathcal{S}_D$ is a set of negative
pairs of instances, and $\ell(\mathbf{s}_1,\mathbf{s}_2,\mathbf{s}_3,\mathbf{s}_4)$ is a loss function that penalizes cases in which a negative pair $(\mathbf{s}_3,\mathbf{s}_4)$ has smaller distance than a positive pair $(\mathbf{s}_1,\mathbf{s}_2)$. A straightforward linear formulation of the loss would be 
$\ell(\mathbf{s}_1,\mathbf{s}_2,\mathbf{s}_3,\mathbf{s}_4) = d_p(\Psi_{\Gamma}(\mathbf{s}_1),\Psi_{\Gamma}(\mathbf{s}_2))-d_p(\Psi_{\Gamma}(\mathbf{s}_3),\Psi_{\Gamma}(\mathbf{s}_4))$. However, only pairs of pairs that violate the distance criterion should contribute 
to the loss, leading to $\ell(\mathbf{s}_1,\mathbf{s}_2,\mathbf{s}_3,\mathbf{s}_4) = \max(0,d_p(\Psi_{\Gamma}(\mathbf{s}_1),\Psi_{\Gamma}(\mathbf{s}_2))-d_p(\Psi_{\Gamma}(\mathbf{s}_3),\Psi_{\Gamma}(\mathbf{s}_4)))$. 
We further replace this loss by a smooth upper bound using log-sum-exp, leading to our final Wasserstein-based loss function
\begin{equation}
\label{eq:final_loss}
\mathcal{L} = \sum_{(\mathbf{s_1},\mathbf{s_2}) \in \mathcal{P}}\sum_{\substack{(\mathbf{s}_3,\mathbf{s}_4) \in \mathcal{N} \\ \mathbf{s}_3\in\{\mathbf{s}_1,\mathbf{s}_2\}}} 
log\left( 1+ exp^{d_p(\Psi_{\Gamma}(\mathbf{s}_1),\Psi_{\Gamma}(\mathbf{s}_2))-d_p(\Psi_{\Gamma}(\mathbf{s}_3),\Psi_{\Gamma}(\mathbf{s}_4))}\right).
\end{equation}
Equation~\ref{eq:final_loss} is of similar structure as other losses used in the literature, including the angular triplet loss~\citep{wang2017deep}, the lifted structured loss~\citep{oh2016deep}, and the N-pair loss~\citep{sohn2016improved}. 

It remains to specify how positive pairs $\mathcal{P}$ and negative pairs $\mathcal{N}$ are sampled for each stochastic gradient descent step. 
We use the approach of~\cite{sohn2016improved} for generating $\mathcal{P}$ and $\mathcal{N}$, which has been shown to give state-of-the-art performance
in several studies~\citep{sohn2016improved,wu2017sampling,yuan2017hard}, in particular outperforming triplet-based sampling~\citep{schroff2015facenet} 
and lifted structure sampling~\citep{oh2016deep}.
The approach constructs a batch of size $2N$ (where $N$ is an adjustable parameter) by sampling from the training data 
$N$ pairs of instances $\mathcal{P}=\{(\mathbf{s}_i,\mathbf{s}_i^+)\}_{i=1}^{N}$ from $N$ different classes, such that each pair $(\mathbf{s}_i,\mathbf{s}_i^+)$ 
is a positive pair from a different class. 
From the sampled batch, a set of $N(N-1)$ negative pairs is constructed by setting $\mathcal{N}=\{(\mathbf{s}_i,\mathbf{s}_j^+)\}_{\substack{i,j=1 \\ j \neq i}}^{N}$.
Note that Equation~\ref{eq:final_loss} can be computed by first computing the embeddings of the $2N$ instances in the batch, and then computing the overall loss. Thus, although the computation is quadratic in $N$, the number of evaluations of the deep neural network model $\Gamma$ is linear in the batch size.

\section{Empirical Study}
\label{sec:empirical_study} 

We empirically study the proposed method in three biometric identification domains involving human eye movements, 
accelerometer-based observation of human gait, and EEG recordings.
As an ablation study, we specifically evaluate which impact the different components of our proposed method -- the metric learning approach, the use of quantile layers to fit the distribution of activations of filters across a sequence, and the Wasserstein-based distance function -- have on overall performance.

\subsection{Data Sets}

We study biometric identification based eye movements, the gait, or the EEG signal of a subject.
In all domains, the data consist of sequential observations of the corresponding low-level sensor signal -- gaze position from an eye tracker, accelerometer measurements,
or EEG measurements -- for different subjects. The task is to identify the subject based on the observed sensor measurements.

The \emph{Dynamic Images and Eye Movements} (DIEM) dataset~\citep{mital2011clustering} contains eye movement data of 210 subjects each viewing a subset of 84 video clips.
The video clips are of varying length with an average of 95 seconds and contain different visual content, such as excerpts from sport matches, documentary videos, movie trailers, or 
recordings of street scenes. The data contain the gaze position on the screen for the left and the right eye, as well as a measurement of the pupil dilation, at a temporal
resolution of 30 Hz. The eye movement data of a particular individual on a particular video clip is thus given by a sequence of six-dimensional vectors (horizontal and vertical 
gaze coordinate for left and right eye plus left and right pupil dilation), that is, $D=6$ in the notation of Section~\ref{sec:quantile_layers}. The average sequence 
length is 2850 and there are 5381 sequences overall.

The gait data we use come from a study by \cite{ihlen2015discriminant} who collected the daily movement activity of 71 subjects for a period of 3 consecutive days. 
The recorded data consists of time series of 3D accelerometer measurements recorded at a sampling rate of 100Hz. For each point in time, the measurement is a $D=6$ dimensional vector consisting of the acceleration and velocity in $x$, $y$, and $z$ direction. 
In the original data set, a continuous measurement for 3 days has been carried out for each individual. These long measurements contain different activities, but also 
long idle periods (for example, during sleep). We concentrate on subsequences showing high activity, by dividing the entire recording for each subject into 
intervals of length one minute, and then selecting for each subject the 30 subsequences that had the largest standard deviation in the 6-dimensional observations. 
This resulted in 2130 sequences overall (30 for each of the 71 subjects), with a length of $T=6000$ per sequence.

The EEG data we use come from a study by~\cite{zhang1995event} who conducted EEG recording sessions with 121 subjects, measuring the signal from 64 electrodes placed on the scalp
at a temporal resolution of 256Hz of the subjects while viewing an image stimulus. The original aim of the study was to find a correlation between EEG observations and genetic predisposition to alcoholism,
but as subject identifiers are available for all recordings the data can also be used in a biometric setting. 
Each subject completed between 40 and 120 trials with 1 second of recorded data per trial. The resulting data
therefore consist of sequences of $D=64$ dimensional vectors with a sequence length of $256$ (one trial for one subject).

\subsection{Problem Setting}
As usual in metric learning, we study a setting in which there are distinct sets of subjects at training and test time. The embedding model is first trained 
on a set of training subjects. On a separate and disjoint set of test subjects, we then evaluate to what degree the learned embedding assigns small distances to 
pairs of test sequences from the same subject, and large distances to pairs of sequences from different subjects. This reflects an application setting
in which new subjects are registered in a database without retraining the embedding model. It also naturally allows the identification of imposters,
that is, subjects who have never been observed (neither during training nor in the database of registered subjects) and try to gain access to the system.

In all three domains, we therefore first split the data into training and test data, such that there is no overlap in subjects between the two. 
For training the embedding model,
we use data of 105 of the 210 subjects (eye movements), 36 of 71 subjects (gait data), or 61 of 121 subjects (EEG data). 
For the eye movement domain, we additionally ensure that there is no overlap in visual stimulus (video clips) between training and test data
by splitting the set of all videos into training and test videos and only keeping the respective sequences in the training and test data. 
During training, each sequence constitutes an instance and the subject its class, and we train either embedding models using metric learning as 
discussed in Section~\ref{sec:distance_function} or, as a baseline, multiclass classification models (see Section~\ref{sec:methods} for details).  
We also set apart the data of 20\% of the training individuals as validation data to tune model hyperparameters. 

At test time, we simulate a biometric application setting by first sampling, for each test subject, a random subset of the sequences available for that subject as instances
that are put in an enrollment database. We then simulate that we observe additional sequences from a subject which are compared to the sequences of all subjects in the enrollment 
database. An embedding is good if the distance between these additional sequences and the enrollment sequences of the same subject is low, compared to the distance to 
the enrollment sequences of other subjects. More precisely, for each subject we use all except five of the sequences available for that subject 
as enrollment sequences. We then study how well the subject can be identified based on observing $n$ of the remaining sequences, for $n\in \{1,..,5\}$.
Given observed sequences $\mathbf{s}_1,...,\mathbf{s}_n$ (representing a subject that is unknown at test time), we compute distances to all subjects $j$ 
as $d_j = \frac{1}{n}\sum_{i=1}^n d(\mathbf{s}_i,\mathbf{s}_{ij})$ where $\mathbf{s}_{ij}$ is the sequence of subject $j$ in the enrollment database with minimal distance to $\mathbf{s}_i$. Here, the definition of the distance function $d$ is method-specific (see below for details).

We first study a \emph{verification} scenario. This is the binary problem of deciding if the observed sequences 
$\mathbf{s}_1,...,\mathbf{s}_n$ match a particular subject $j$, by comparing the computed distance $d_j$ to a threshold value. Varying the threshold trades of 
false-positive versus false-negative classifications, yielding a ROC curve and AUC score. Note that the verification scenario also covers the setting in which 
in imposter is trying to get access to a system as a particular user; the false-positive rate is the rate at which such imposters would be accepted.

We then study a \emph{multiclass identification} scenario, where we use the model to assign the observed sequences $\mathbf{s}_1,...,\mathbf{s}_n$ to a subject enrolled 
in the database (the subject $j^* = \arg\min_j d_j$). This constitutes a multiclass classification problem for which (multiclass) accuracy is measured. 
In this experiment, we also vary the number of subjects under study, by randomly sampling a subset of subjects which are enrolled in the database; the same
subset of subjects is observed at test time. The identification problem becomes more difficult as the number of subjects increases.

We finally study the robustness of the model to imposters in the multiclass identification scenario, an experiment we denote as \emph{multiclass imposters}. 
This reflects applications in which access to a system does not require a user name, because the system tries to automatically identify who is trying to gain access.
In this experiment, half of the test subjects play the role of imposters who are not registered in the enrollment
database. 
As in the multiclass identification setting, observed sequences are matched to the enrolled subject with minimum distance.
This minimum distance is then compared to a threshold value; if the threshold is exceeded, the match is rejected and the observed sequences are classified as
belonging to an imposter. Varying the threshold trades off false-positives (match of imposter accepted) versus false-negatives (match of a subject enrolled in the database
rejected), yielding a ROC curve and AUC. Correctly rejecting imposters is harder in this setting because it suffices for an imposter to successfully impersonate any enrolled subject.
In this experiment we also vary the number of subjects enrolled in the database.

In all three scenarios, the split of sequences into enrollment and observed sequences is repeated 10 times to obtain standard deviations of results. 
Moreover, accuracies and AUCs will increase with increasing $n$, as identification becomes easier the more data of an unknown subject is available.

\subsection{Methods Under Study}
\label{sec:methods}
We generally study the deep convolutional architecture proposed by~\cite{abdelwahab2019Quantile} for biometric identification, which consists of 16 stacked 1D-convolution layers
with PReLU activation functions.
We vary the aggregation operation, loss function, and training algorithm in order to evaluate the impact of these components on overall performance.

\noindent\textbf{\ourmethod}: Our method, combining the quantile embeddings of Section~\ref{sec:quantile_layers} with the Wasserstein-based loss function and metric learning algorithm
of Section~\ref{sec:distance_function}. In all experiments, we set the parameter $p$ of the distance function (see Definition~\ref{def:distance}) to one, that is, we use
the Earth Mover Distance variant of the Wasserstein distance. The convolutional neural network architecture $\Gamma$ of Section~\ref{sec:quantile_layers} is given by 
16 stacked convolution layers with parametric RELU activations as defined by~\cite{abdelwahab2019Quantile}. The number of sampling points for the quantile function is $M=16$.
At test time, distance between instances is given by the distance function from Definition~\ref{def:distance}.

\noindent\textbf{\quantilenpl}: This method uses the same network architecture and quantile embedding as \ourmethod. However, the resulting quantile embedding is then flattened
into an $K\cdot M$ vector embedding, with entries $\bar{Q}_{\Gamma_k(\mathbf{s})}(\alpha_m)$ for $k \in \{1,...,K\}$ and $m\in\{1,...,M\}$. Then standard $N$-pair loss, which is based on cosine similarities between embedding vectors~\citep{sohn2016improved}, is used for training.
At test time, the distance between instances is given by negative cosine similarity.
This method utilizes quantile-based aggregation and metric learning, but does not employ our Wasserstein-based loss function.  

\noindent\textbf{\maxnpl}: This method uses the same basic network architecture as \quantilenpl, but uses standard max-pooling instead of a quantile layer for global aggregation. 
This results in a $K$-dimensional embedding vector. As for \quantilenpl, the model is trained using metric learning with the $N$-pair loss. At test time,
distance is given by negative cosine similarity. 
This baseline uses metric learning, but neither quantile layers nor the Wasserstein-based loss function.

\noindent\textbf{\quantileclass}: This baseline uses the same network architecture and flattened quantile
embedding as \quantilenpl, but feeds the flattened embedding vector into a dense classification layer with softmax activation.
The models is trained in a classification setting using multiclass crossentropy. 
Distance at test time is given by negative cosine similarity.
This model is identical to the model presented in~\cite{abdelwahab2019Quantile}, except that we remove the final classification layer at test time to generate embeddings
for novel subjects.

For all methods, training is carried out using the Adam optimizer with learning rate 0.0001 for 50000 iterations, and the regularizer of the PReLU activation function is tuned
 as a hyperparameter on the validation set as in~\citep{abdelwahab2019Quantile}.

\begin{table}[t!]
\begin{center}
\setlength\tabcolsep{2pt}
\resizebox{\textwidth}{!}{
\begin{tabular}{llllll}
\hline\noalign{\smallskip}
Eye data&	1 Video&	2 Videos&	3 Videos&	4 Videos&	5 Videos\\
\noalign{\smallskip}\hline\noalign{\smallskip}
\ourmethod &\textbf{0.9466}$\pm$0.0032&\textbf{0.9716}$\pm$0.0020&\textbf{0.9799}$\pm$0.0013&\textbf{0.9837}$\pm$0.0008&\textbf{0.9860}$\pm$0.0005\\
\quantilenpl&0.9345$\pm$0.0033&0.9584$\pm$0.0027&0.9667$\pm$0.0020&0.9705$\pm$0.0014&0.9738$\pm$0.0010\\
\maxnpl&0.8890$\pm$0.0035&0.9232$\pm$0.0028&0.9334$\pm$0.0017&0.9392$\pm$0.0014&0.9437$\pm$0.0016\\
\quantileclass&0.9007$\pm$0.0053&0.9318$\pm$0.0029&0.9424$\pm$0.0025&0.9503$\pm$0.0025&0.9538$\pm$0.0026\\
\hline\noalign{\smallskip}
Gait data&1 Minute&2 Minutes&3 Minutes&4 Minutes&5 Minutes\\
\noalign{\smallskip}\hline\noalign{\smallskip}
\ourmethod&\textbf{0.9923}$\pm$0.0008&\textbf{0.9963}$\pm$0.0003&\textbf{0.9971}$\pm$0.0003&\textbf{0.9974}$\pm$0.0002&\textbf{0.9978}$\pm$0.0001\\
\quantilenpl&0.9889$\pm$0.0009&0.9932$\pm$0.0004&0.9943$\pm$0.0003&0.9947$\pm$0.0002&0.9951$\pm$0.0002\\
\maxnpl&	0.9459$\pm$0.0027&0.9624$\pm$0.0027&0.9690$\pm$0.0021&0.9735$\pm$0.0016&0.9757$\pm$0.0012\\
\quantileclass&0.9579$\pm$0.0040&0.9756$\pm$0.0018&0.9812$\pm$0.0016&0.9856$\pm$0.0011&0.9878$\pm$0.0008\\
\hline\noalign{\smallskip}
EEG data&1 Second&2 Seconds&3 Seconds&4 Seconds&5 Seconds\\
\noalign{\smallskip}\hline\noalign{\smallskip}
\ourmethod&\textbf{0.9968}$\pm$0.0006&\textbf{0.9985}$\pm$0.0001&\textbf{0.9988}$\pm$0.0001&\textbf{0.9991}$\pm$0.0000&\textbf{0.9992}$\pm$0.0000\\
\quantilenpl&0.9927$\pm$0.0005&0.9941$\pm$0.0005&0.9953$\pm$0.0003&0.9955$\pm$0.0002&0.9959$\pm$0.0001\\
\maxnpl&0.9611$\pm$0.0012&0.9687$\pm$0.0005&0.9713$\pm$0.0005&0.9722$\pm$0.0005&0.9732$\pm$0.0005\\
\quantileclass&0.9796$\pm$0.0017&0.9868$\pm$0.0009&0.9901$\pm$0.0010&0.9920$\pm$0.0006&0.9923$\pm$0.0007\\
\noalign{\smallskip}\hline
\end{tabular}
}
\end{center}
\caption{
Area under the ROC curve with standard error for all methods and domains in the verification setting for varying number $n \in \{1,2,3,4,5\}$ of observed sequences.
}
\label{tab:res}
\end{table}

\subsection{Results}

We present and discuss empirical results for the different domains in turn.

\subsubsection{Eye Movements}
Table~\ref{tab:res}, upper third, shows area under the ROC curve for all methods and varying number $n$ of observed sequences in the eye movement domain.
Comparing \ourmethod and \quantilenpl, we observe that the Wasserstein-based loss introduced in Section~\ref{sec:distance_function}, which works on the distributional 
embedding given by the piecewise linear approximations of the quantile functions, clearly outperforms flattening the distributional embedding and using $N$-pair loss.
Comparing \maxnpl with \quantilenpl and \ourmethod shows that using quantile layers improves accuracy compared to max-pooling even if the quantile embedding is flattened (and more so
if Wasserstein-based loss is used).
Classification training (\quantileclass) reduces accuracy compared to metric learning (\quantilenpl).
As expected, AUC increases with the number $n$ of sequences observed at test time.
Figure~\ref{fig:eyev} shows ROC curves in the verification setting for $n=5$. 

\begin{figure}[t]
\centering
\includegraphics[width=0.45\linewidth]{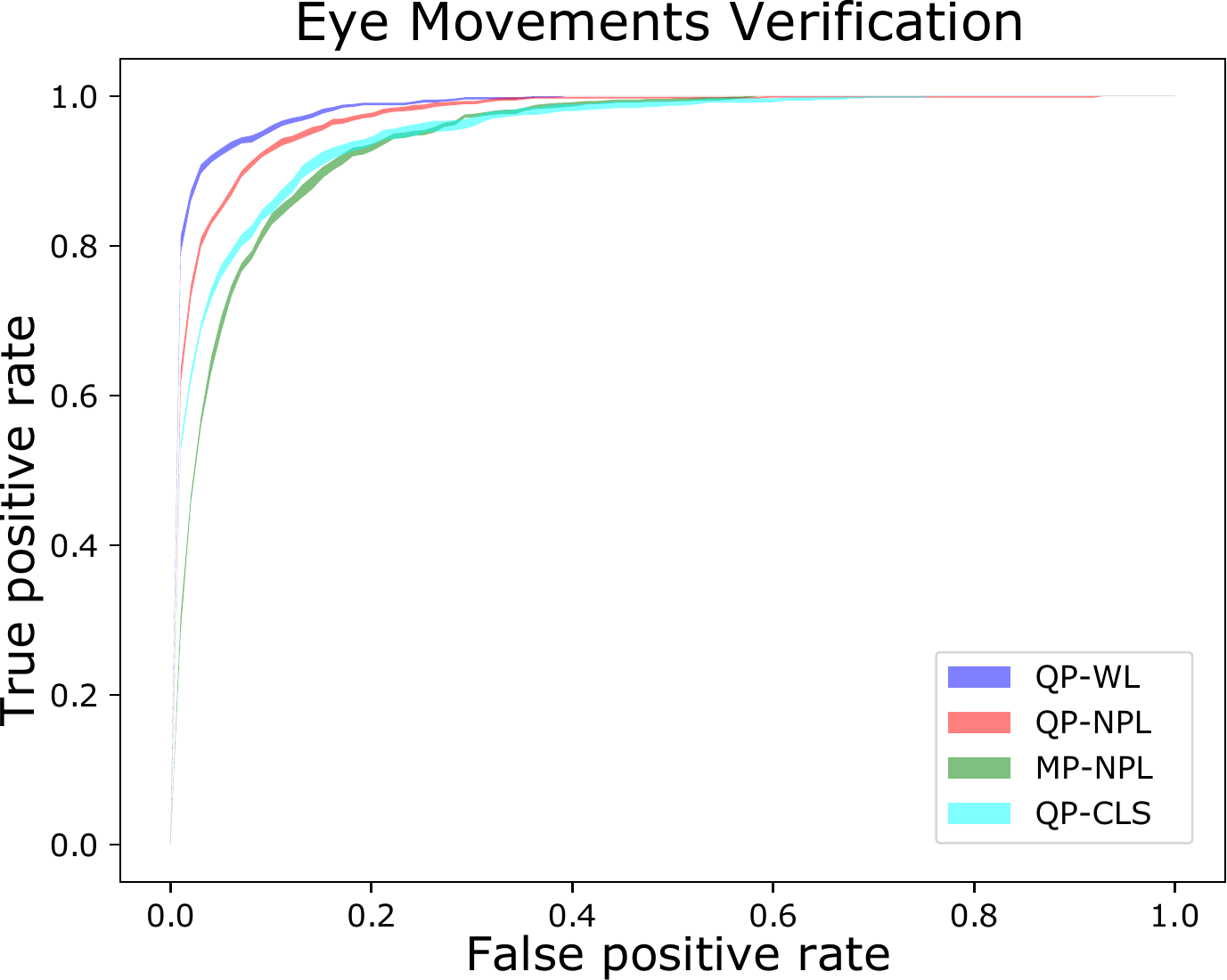}
\caption{
ROC curves in the eye movement domain for all methods using $n=5$ observed sequences. Shaded region in ROC curves indicates standard error. 
}
\label{fig:eyev}
\end{figure}

\begin{figure}[t]
\centering
\includegraphics[width=0.45\linewidth]{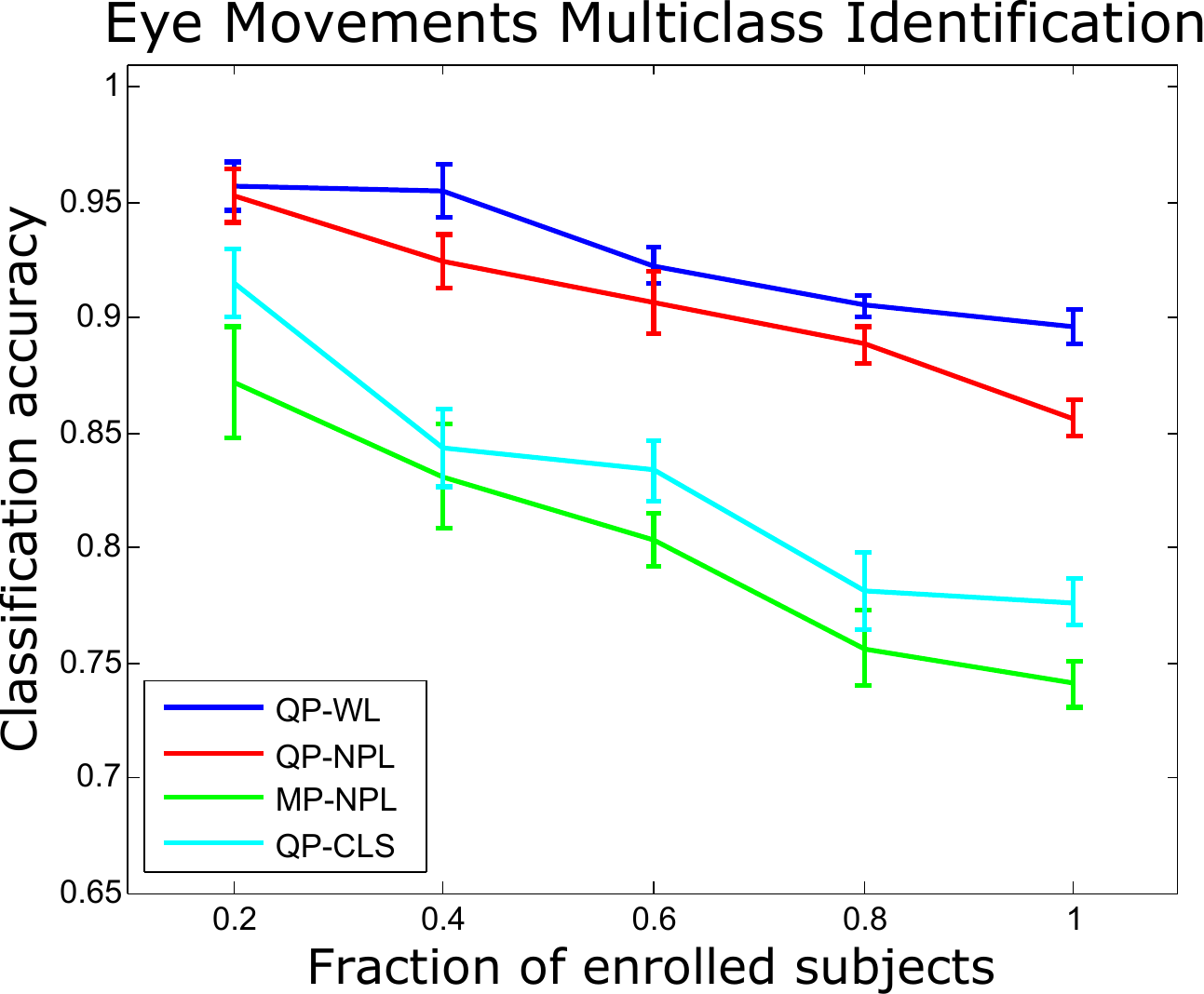}
\hspace{0.3cm}
\includegraphics[width=0.45\linewidth]{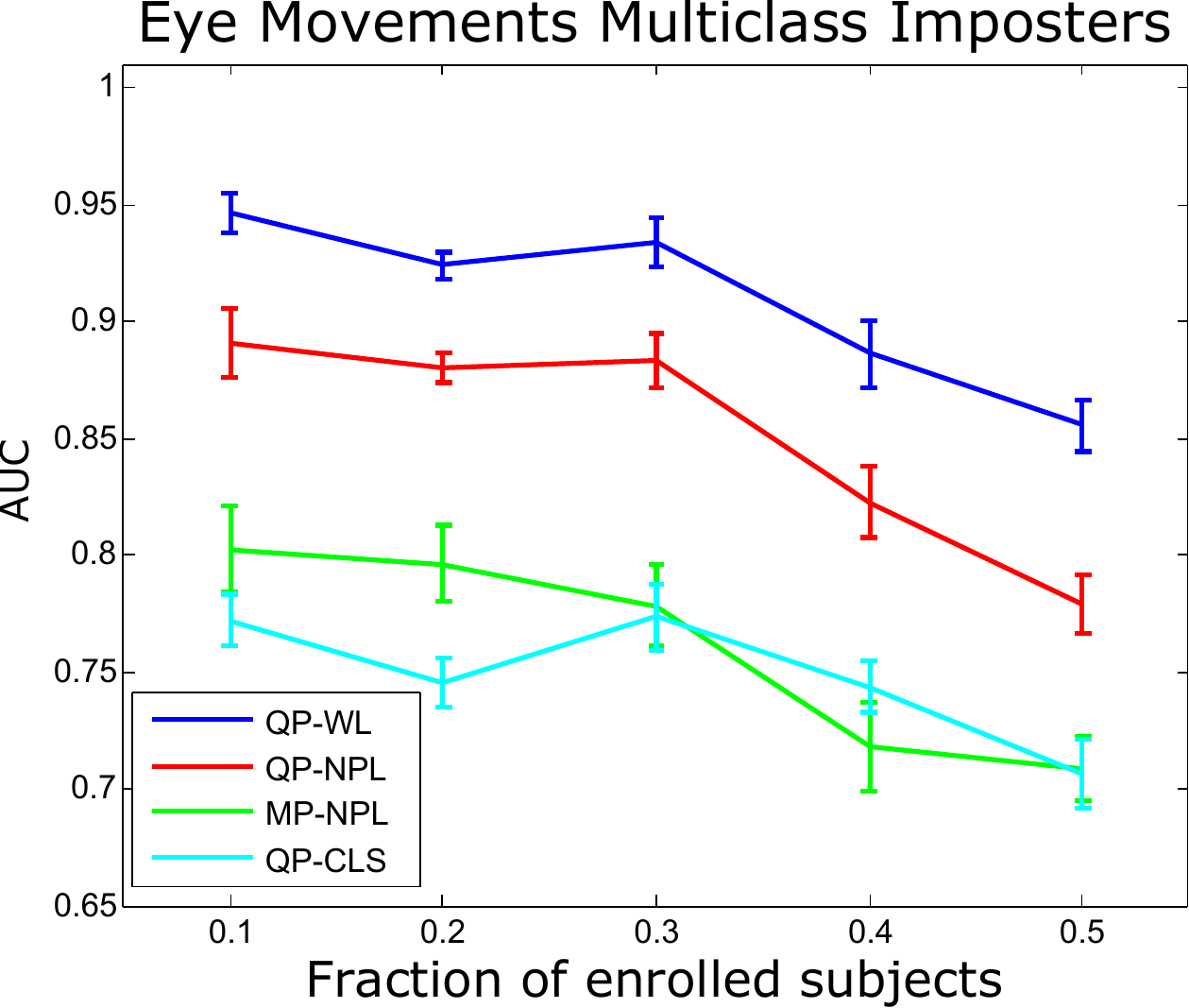}
\caption{
Left: Identification accuracy in the multiclass identification scenario for the eye movement domain and $n=5$ observed test instances as a function of the fraction of subjects that are enrolled.
Right: area under the ROC curve for multiclass imposters as a function of the fraction of subjects enrolled. In the imposter scenario, 50\% of subjects are imposters and therefore never enrolled.
Error bars indicate the standard error.
}
\label{fig:eyei}
\end{figure}

Figure~\ref{fig:eyei} (left) shows multiclass identification accuracy for $n=5$ observed sequences as a function of the fraction of the 105 subjects who are enrolled. 
Relative results for the different methods are similar as in the verification setting. Accuracy decreases slightly when more subjects are enrolled, as the multiclass
problem becomes more difficult. Figure~\ref{fig:eyei} (right) shows the robustness of the model to multiclass imposters as a function of the fraction of the 105 subjects who are enrolled
(up to 50\%, as half of the subjects are imposters). We observe that \ourmethod is much more robust to imposters than the baseline methods.

\begin{figure}[t]
\centering
\includegraphics[width=0.312\linewidth]{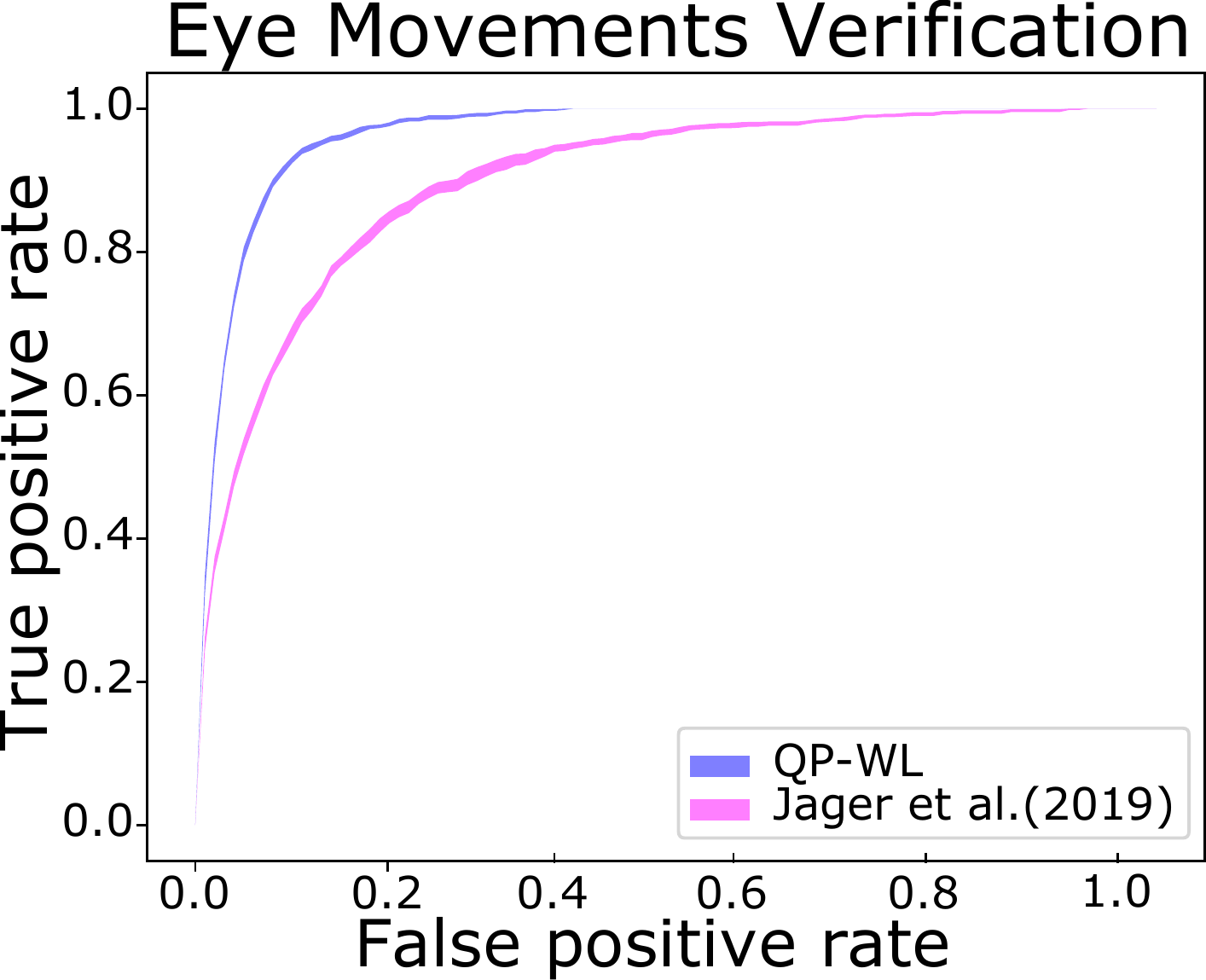}
\hspace{2mm}
\includegraphics[width=0.30\linewidth]{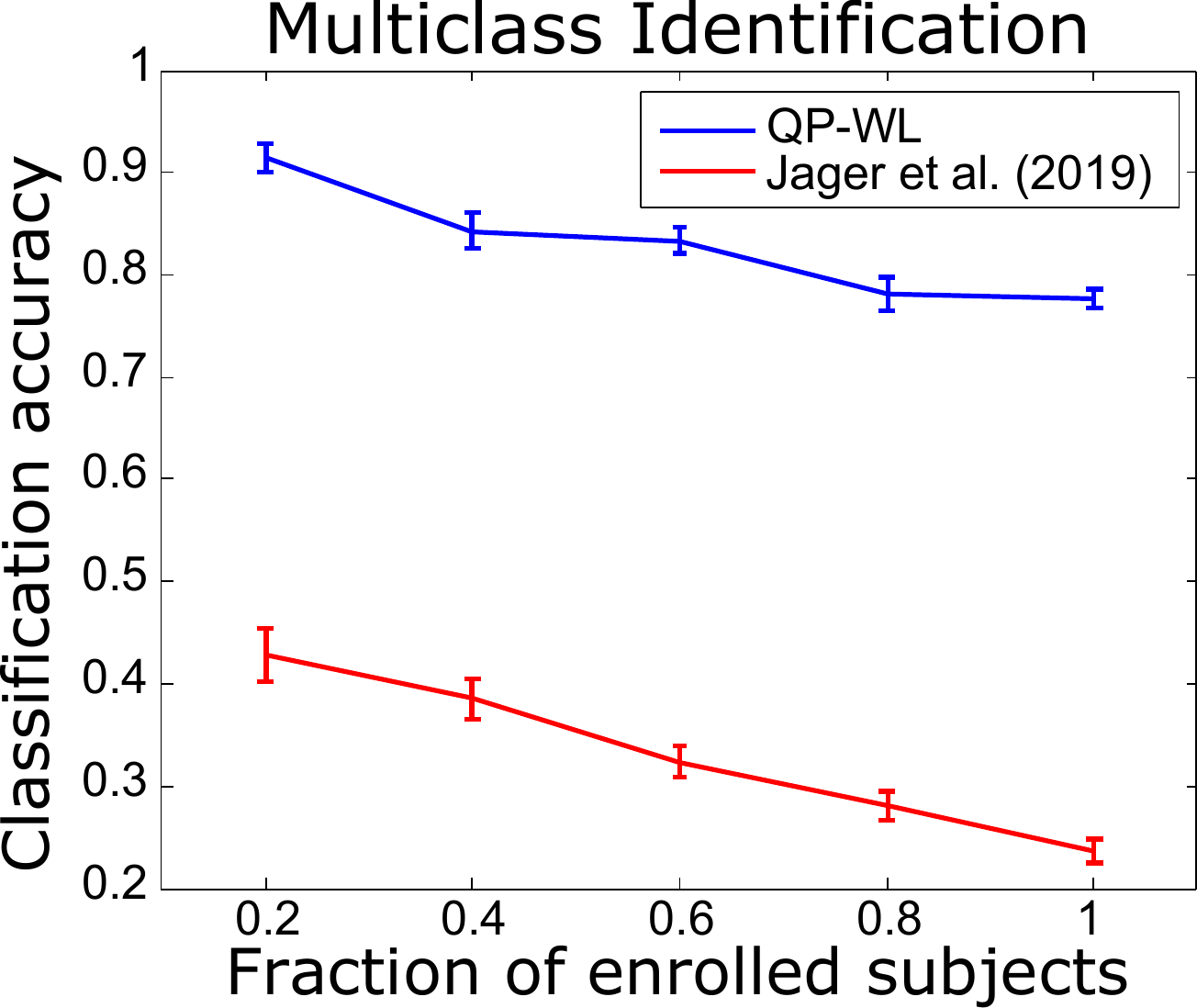}
\hspace{2mm}
\includegraphics[width=0.30\linewidth]{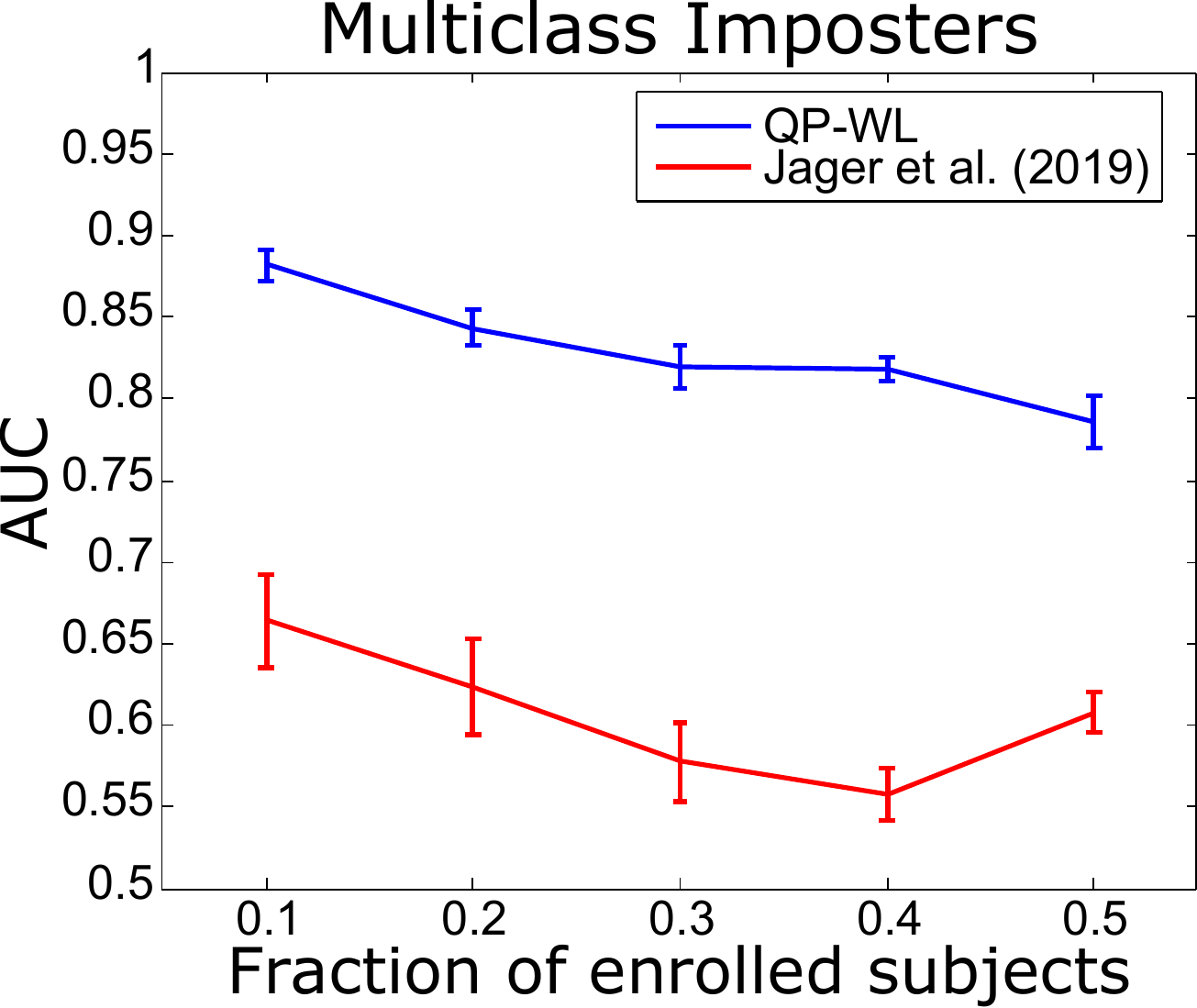}
\caption{
Comparison between \ourmethod and \jager in the eye movement domain: area under ROC curve in verification scenario (left), identification accuracy in multiclass identification scenario (center), and robustness of model to multiclass imposters (right). 
In this experiment, the data is simplified for both methods to match the requirements of \jager, see text for details. 
Results of \ourmethod therefore differ from results presented in Figure~\ref{fig:eyev} and Figure~\ref{fig:eyei}. Error bars indicate the standard error.
}
\label{fig:soa}
\end{figure}

In the eye movement domain, we also compare against the state-of-the-art model by~\cite{jager2019eyedentification}, denoted \jager. 
\jager uses angular gaze velocities averaged over left and right eye as input, which we compute from our raw data. 
We replicate the setting of~\cite{jager2019eyedentification} by training the model using multiclass classification and using the last layer before 
the classification layer as the embedding at test time. The \jager architecture cannot deal with variable-lenght
sequences, we therefore split the variable-length sequences in our data into shorter sequences of fixed length, namely the length of the shortest sequence (27 seconds). 
For a fair comparison, we also simplify the data for our model in this experiment: using only the average gaze point rather than left and right gaze point separately,
removing pupil dilation, and using the same fixed-length sequences. Figure~\ref{fig:soa} shows ROC curves for the verification scenario (left) and identification accuracy
(center) as well as AUC in the imposter scenario for our model \ourmethod on the simplified data and \jager. Comparing to Figure~\ref{fig:eyev} and Figure~\ref{fig:eyei}
we observe that accuracies are reduced for our model by using the simplified data, but the model still outperforms \jager by a wide margin.
We note that the model of~\cite{jager2019eyedentification} is focused on microsaccades, which are likely not detectable in our data due to the low temporal 
resolution (30Hz compared to 1000Hz in the study by~\cite{jager2019eyedentification}), which might explain the relatively poor performance of the model on our data.

\subsubsection{Gait}

Table~\ref{tab:res}, center third, shows area under the ROC curve for all methods and varying number $n$ of observed sequences in the gait domain. We observe the ordering in terms of relative performance between the different methods
as in the eye movements domain, with clear benefits when using the proposed loss function based on Wasserstein distance (\ourmethod versus \quantilenpl), when using quantile layers instead of max-pooling 
aggregation (\ourmethod and \quantilenpl versus \maxnpl), and when using metric learning rather than classification-based training (\quantilenpl versus \quantileclass). Figure~\ref{fig:gaitv} shows ROC curves for verification 
at $n=5$ in the gait domain. Figure~\ref{fig:gaiti} (left) shows identification accuracy as a function of the fraction of subjects enrolled in the gait domain; in this setup the ordering of methods in terms of performance is the same
but the difference between \ourmethod and \quantilenpl less pronounced. Figure~\ref{fig:gaiti} (right) shows robustness to multiclass imposters, with again a clear advantage of \ourmethod over the baselines.

\begin{figure}[t]
\centering
\includegraphics[width=0.45\linewidth]{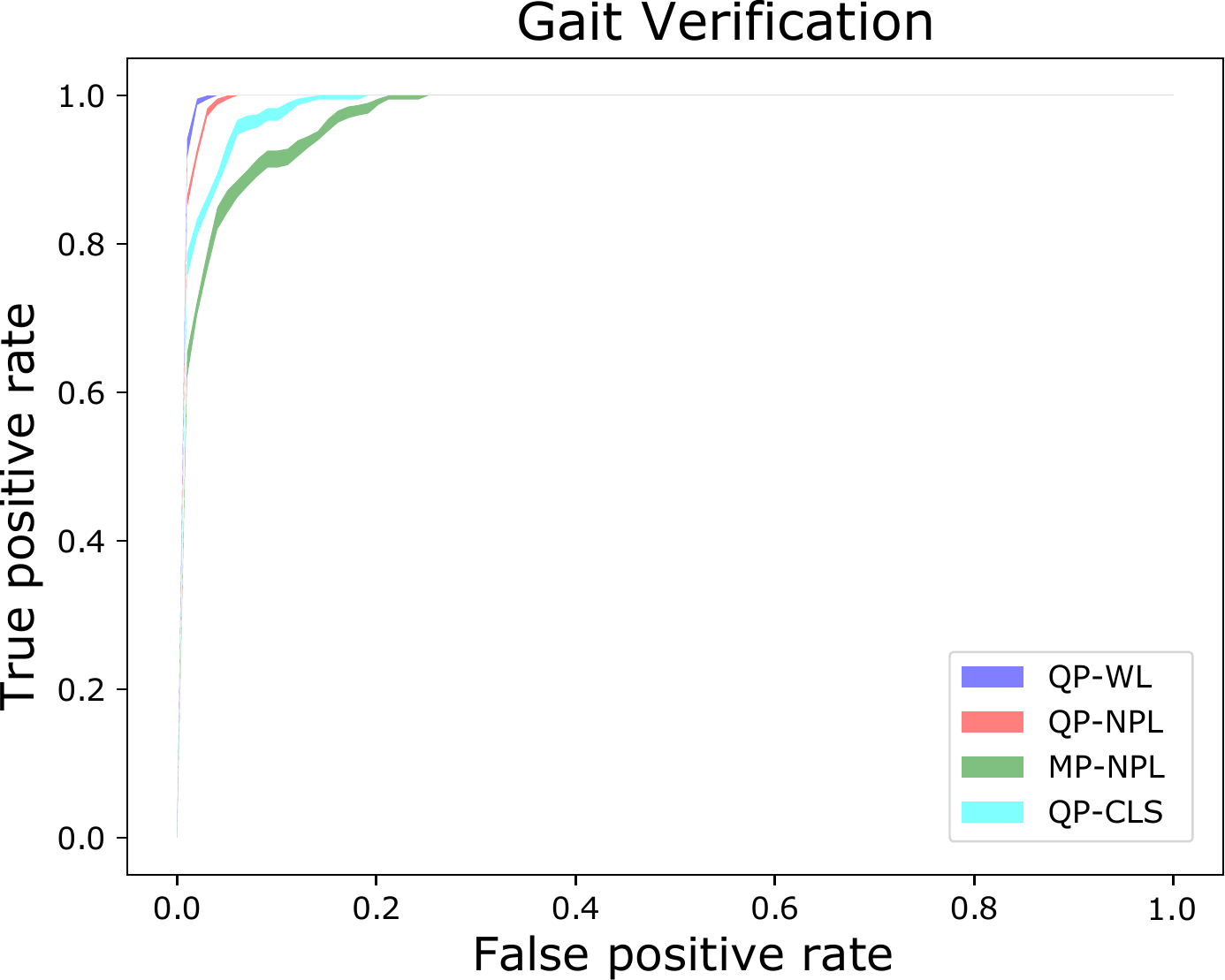}
\caption{
ROC curves in the gait domain for all methods using $n=5$ observed sequences. Shaded region in ROC curves indicates standard error. 
}
\label{fig:gaitv}
\end{figure}

\begin{figure}[t]
\centering
\includegraphics[width=0.45\linewidth]{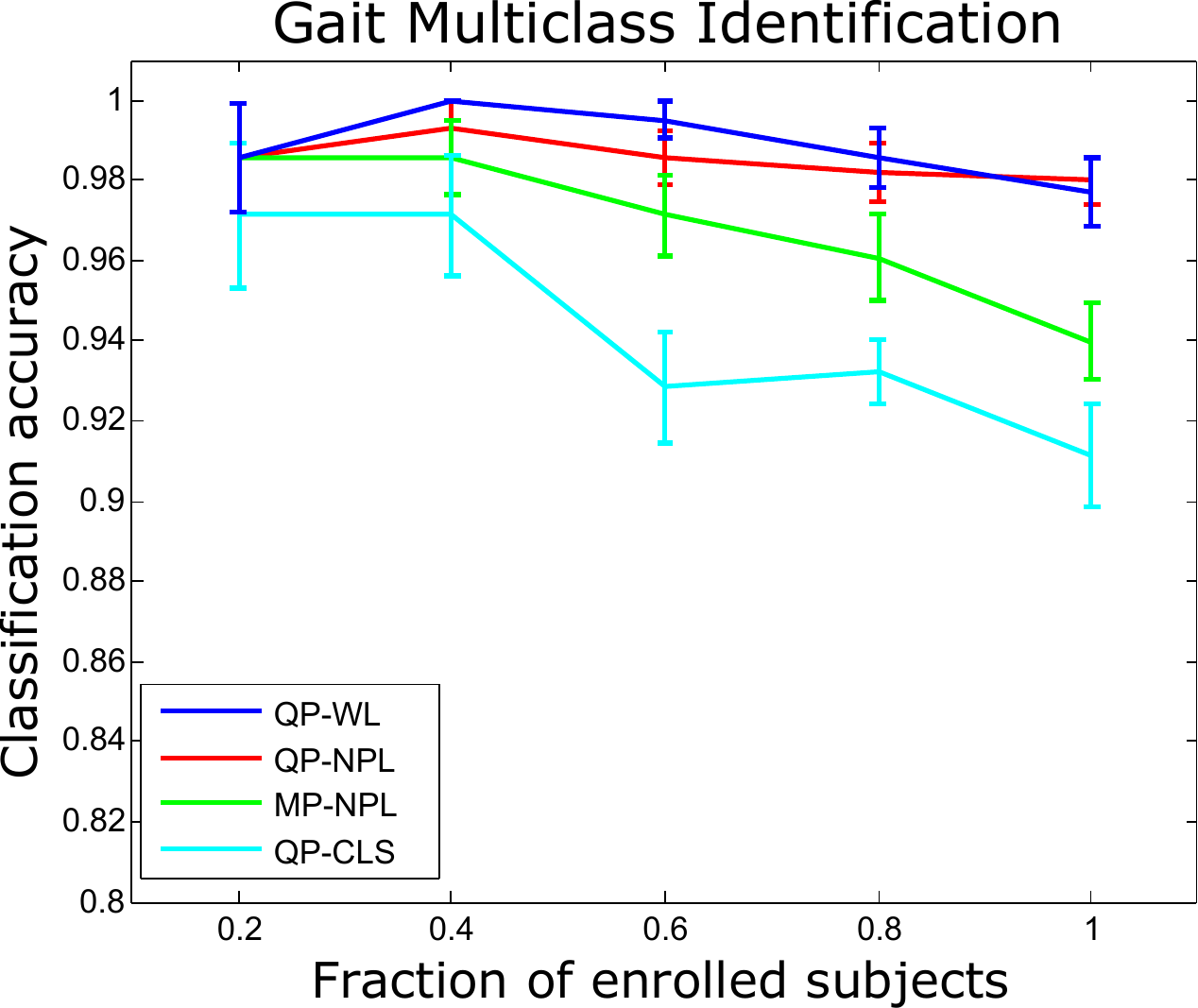}
\hspace{0.3cm}
\includegraphics[width=0.45\linewidth]{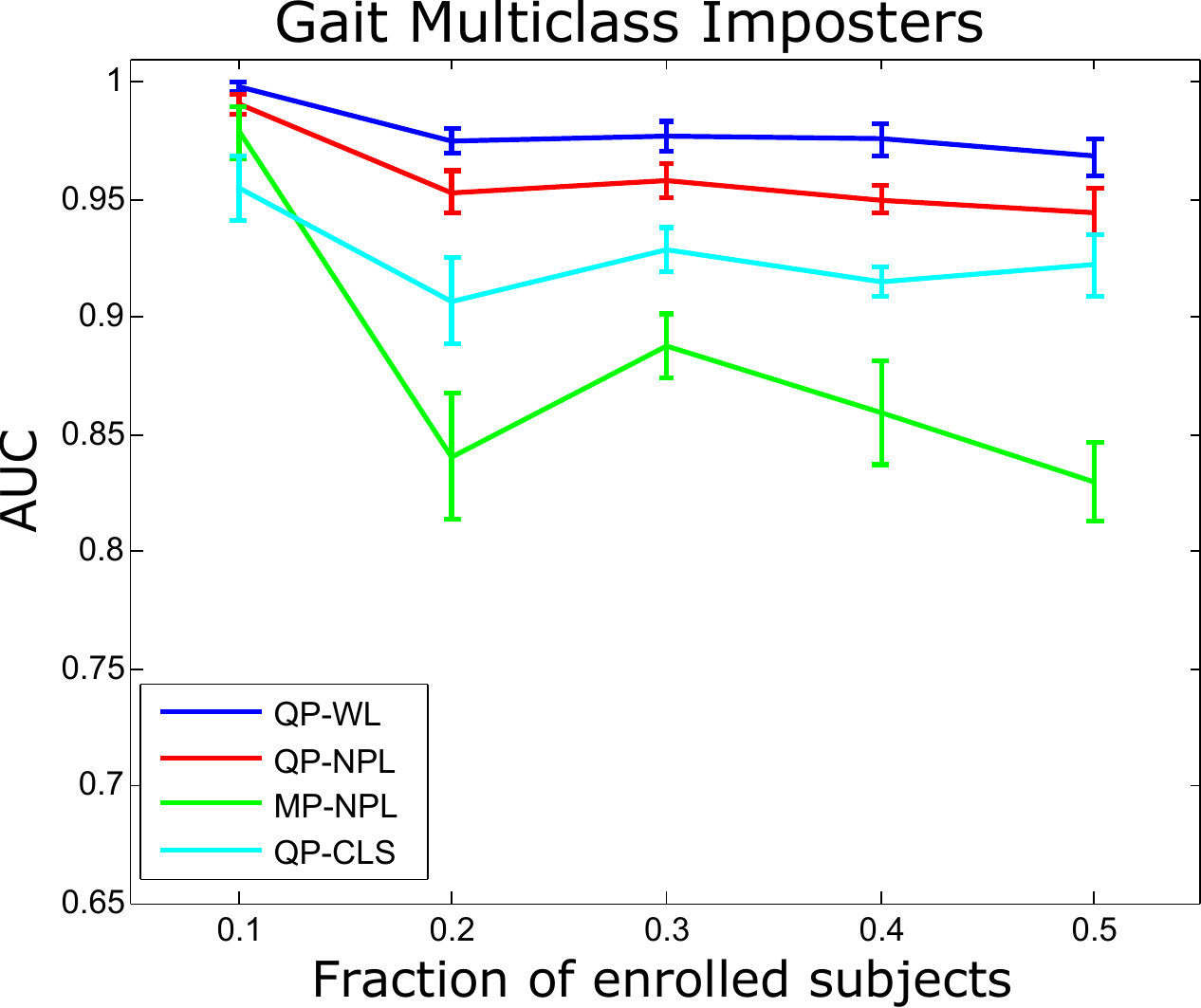}
\caption{
Left: Identification accuracy in the multiclass identification scenario for the gait domain and $n=5$ observed test instances as a function of the fraction of subjects that are enrolled.
Right: area under the ROC curve for multiclass imposters as a function of the fraction of subjects enrolled. In the imposter scenario, 50\% of subjects are imposters and therefore never enrolled.
Error bars indicate the standard error.
}
\label{fig:gaiti}
\end{figure}

\subsubsection{EEG}

\begin{figure}[t]
\centering
\includegraphics[width=0.45\linewidth]{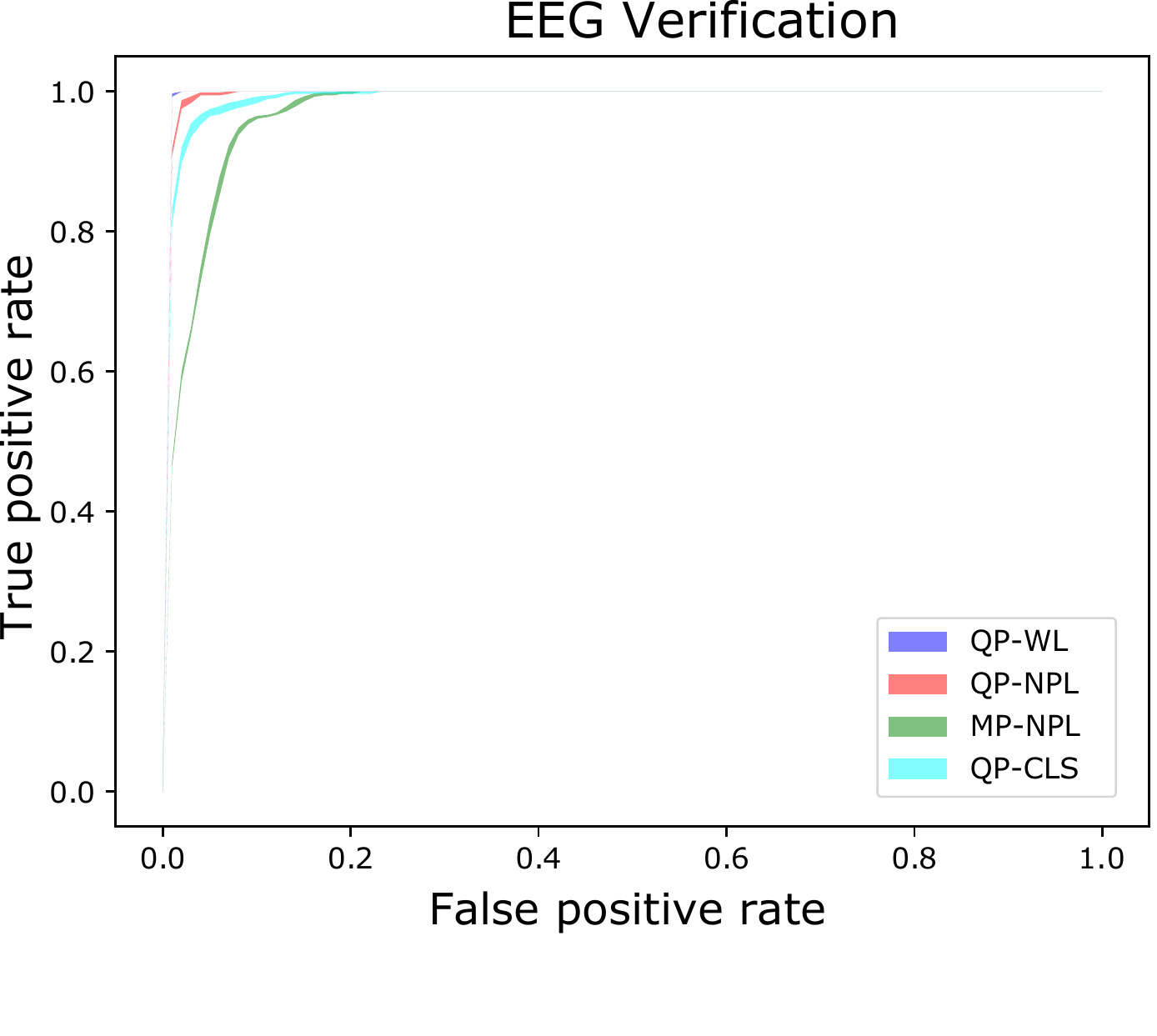}
\caption{
ROC curves in the EEG domain for all methods using $n=5$ observed sequences. Shaded region in ROC curves indicates standard error. 
}
\label{fig:eegv}
\end{figure}

\begin{figure}[t]
\centering
\includegraphics[width=0.45\linewidth]{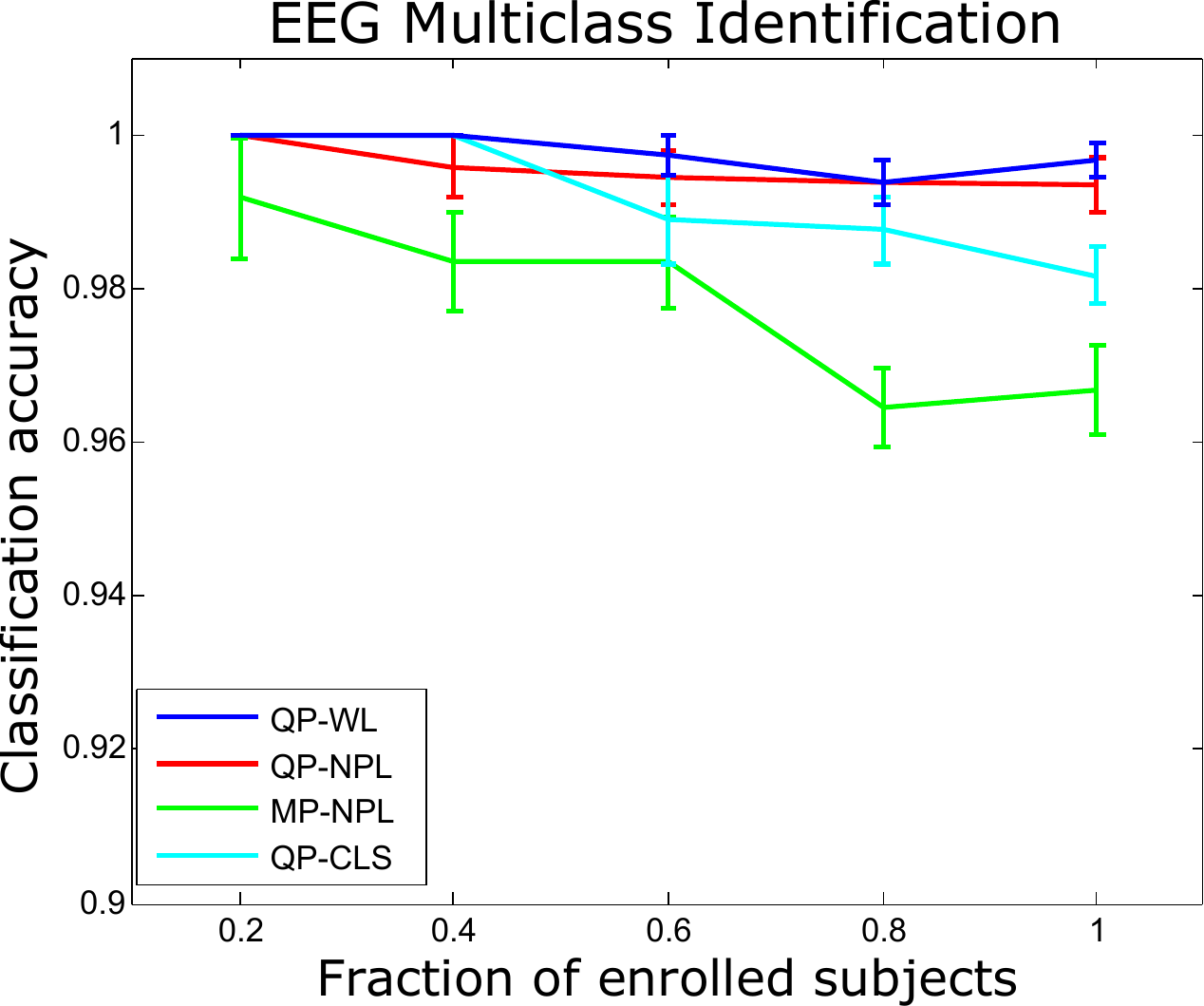}
\hspace{0.3cm}
\includegraphics[width=0.45\linewidth]{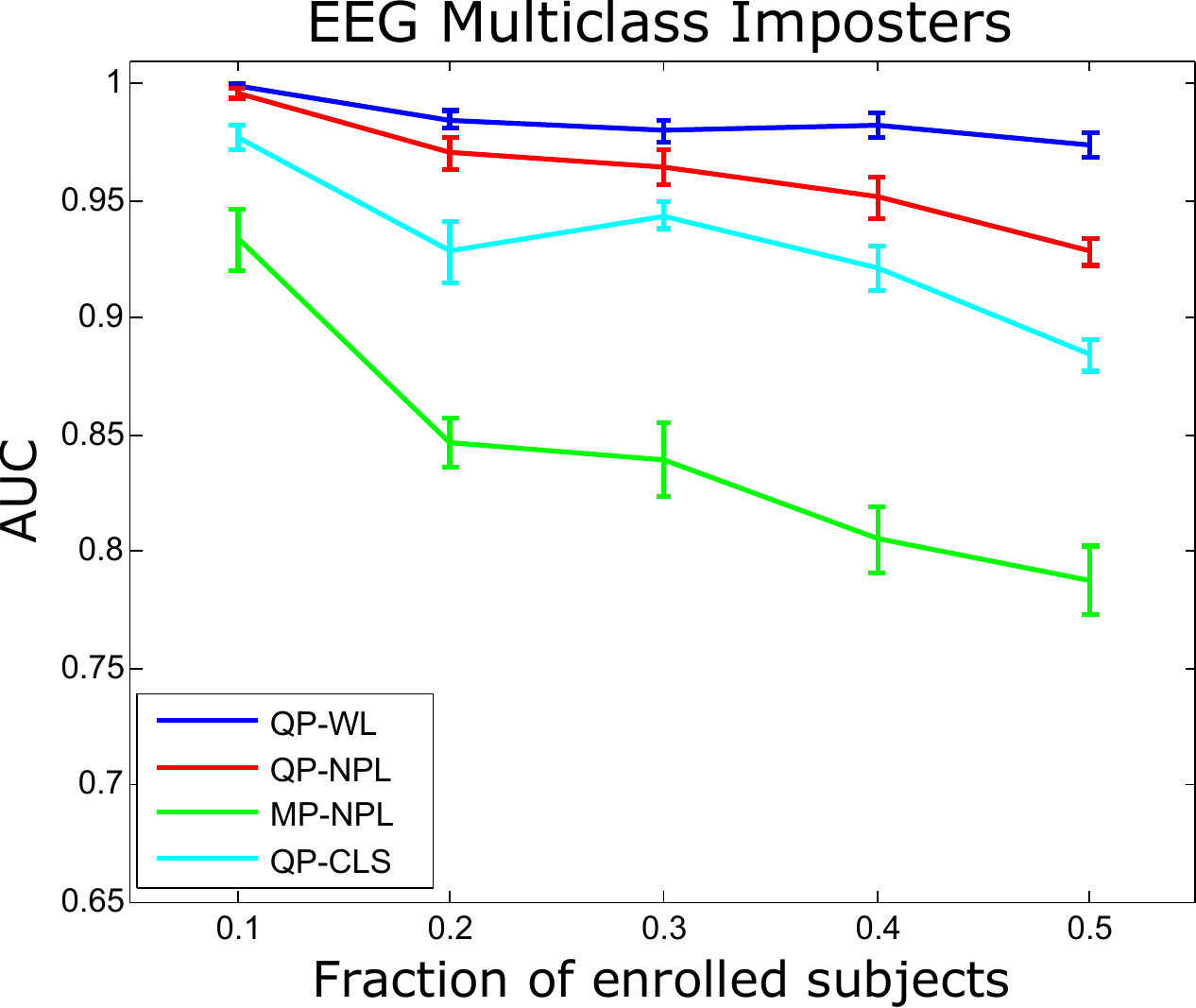}
\caption{
Left: Identification accuracy in the multiclass identification scenario for the EEG domain and $n=5$ observed test instances as a function of the fraction of subjects that are enrolled.
Right: area under the ROC curve for multiclass imposters as a function of the fraction of subjects enrolled. In the imposter scenario, 50\% of subjects are imposters and therefore never enrolled.
Error bars indicate the standard error.
}
\label{fig:eegi}
\end{figure}

Table ~\ref{tab:res}, bottom third, shows area under the ROC curve for all methods and varying number $n$ of observed test sequences in the EEG domain. Relative performance of methods is generally similar as in the other two
domains. \ourmethod clearly outperforms the closest baseline, reducing 1-AUC by between 56\% ($n=1$) and 80\% ($n=5)$. Figure~\ref{fig:eegv} shows ROC curves in the verification setting. 
Figure~\ref{fig:eegi} (left) and Figure~\ref{fig:eegi} (right) show identification accuracy as a function of the fraction of subjects enrolled and robustness of the models to multiclass imposters.
As in the gait domain, differences are more pronounced in the latter setting.

\section{Conclusions}

We developed a model for distributional embeddings of variable-length sequences using deep neural networks. Building on existing work on quantile layers, the
model represents an instance by the distribution of the learned deep features across the sequence. We developed a distance function for these distributional embeddings based on the
Wasserstein distance between the corresponding distributions, and from this distance function a loss function for performing metric learning with the proposed model. 
A key point about the model is end-to-end learnability: by using piecewise linear approximations of the quantile functions, and based on those providing a closed-form 
solution for the Wasserstein distance, gradients can be traced through the embedding and loss calculations. 
In our empirical study, distributional embeddings outperformed standard vector embeddings by a large margin on three data sets from different domains.


\begin{thebibliography}{32}
\providecommand{\natexlab}[1]{#1}
\providecommand{\url}[1]{{#1}}
\providecommand{\urlprefix}{URL }
\expandafter\ifx\csname urlstyle\endcsname\relax
  \providecommand{\doi}[1]{DOI~\discretionary{}{}{}#1}\else
  \providecommand{\doi}{DOI~\discretionary{}{}{}\begingroup
  \urlstyle{rm}\Url}\fi
\providecommand{\eprint}[2][]{\url{#2}}

\bibitem[{Abdelwahab and Landwehr(2019)}]{abdelwahab2019Quantile}
Abdelwahab A, Landwehr N (2019) Quantile layers: Statistical aggregation in
  deep neural networks for eye movement biometrics. In: Proceedings of the 30th
  European Conference on Machine Learning

\bibitem[{Arjovsky et~al(2017)Arjovsky, Chintala, and
  Bottou}]{arjovsky2017wasserstein}
Arjovsky M, Chintala S, Bottou L (2017) Wasserstein generative adversarial
  networks. In: International conference on machine learning, pp 214--223

\bibitem[{Athiwaratkun and Wilson(2017)}]{athiwaratkun2017multimodal}
Athiwaratkun B, Wilson A (2017) Multimodal word distributions. In: Proceedings
  of the 55th Annual Meeting of the Association for Computational Linguistics
  (Volume 1: Long Papers), pp 1645--1656

\bibitem[{Bojchevski and G{\"u}nnemann(2018)}]{bojchevski2017deep}
Bojchevski A, G{\"u}nnemann S (2018) Deep {Gaussian} embedding of graphs:
  Unsupervised inductive learning via ranking. In: International Conference on
  Learning Representations, pp 1--13

\bibitem[{Bucher et~al(2016)Bucher, Herbin, and Jurie}]{bucher2016improving}
Bucher M, Herbin S, Jurie F (2016) Improving semantic embedding consistency by
  metric learning for zero-shot classiffication. In: European Conference on
  Computer Vision, Springer, pp 730--746

\bibitem[{Cambanis et~al(1976)Cambanis, Simons, and
  Stout}]{cambanis1976inequalities}
Cambanis S, Simons G, Stout W (1976) Inequalities for ek (x, y) when the
  marginals are fixed. Zeitschrift f{\"u}r Wahrscheinlichkeitstheorie und
  verwandte Gebiete 36(4):285--294

\bibitem[{Chung et~al(2018)Chung, Nagrani, and Zisserman}]{chung2018voxceleb2}
Chung JS, Nagrani A, Zisserman A (2018) Voxceleb2: Deep speaker recognition.
  Proc Interspeech 2018 pp 1086--1090

\bibitem[{Frogner et~al(2015)Frogner, Zhang, Mobahi, Araya, and
  Poggio}]{frogner2015learning}
Frogner C, Zhang C, Mobahi H, Araya M, Poggio TA (2015) Learning with a
  {Wasserstein} loss. In: Advances in Neural Information Processing Systems, pp
  2053--2061

\bibitem[{Gao and Kleywegt(2016)}]{gao2016distributionally}
Gao R, Kleywegt AJ (2016) Distributionally robust stochastic optimization with
  {Wasserstein} distance. arXiv preprint arXiv:160402199

\bibitem[{Gibiansky et~al(2017)Gibiansky, Arik, Diamos, Miller, Peng, Ping,
  Raiman, and Zhou}]{gibiansky2017deep}
Gibiansky A, Arik S, Diamos G, Miller J, Peng K, Ping W, Raiman J, Zhou Y
  (2017) Deep voice 2: Multi-speaker neural text-to-speech. In: Advances in
  neural information processing systems, pp 2962--2970

\bibitem[{Hadsell et~al(2006)Hadsell, Chopra, and
  LeCun}]{hadsell2006dimensionality}
Hadsell R, Chopra S, LeCun Y (2006) Dimensionality reduction by learning an
  invariant mapping. In: 2006 IEEE Computer Society Conference on Computer
  Vision and Pattern Recognition, IEEE, vol~2, pp 1735--1742

\bibitem[{Ihlen et~al(2015)Ihlen, Weiss, Helbostad, and
  Hausdorff}]{ihlen2015discriminant}
Ihlen EA, Weiss A, Helbostad JL, Hausdorff JM (2015) The discriminant value of
  phase-dependent local dynamic stability of daily life walking in older adult
  community-dwelling fallers and nonfallers. BioMed research international

\bibitem[{Jager et~al(2019)Jager, Makowski, Prasse, Liehr, Seidler, and
  Scheffer}]{jager2019eyedentification}
Jager L, Makowski S, Prasse P, Liehr S, Seidler M, Scheffer T (2019) Deep
  eyedentification: Biometric identification using micro-movements of the eye.
  In: Proceedings of the 30th European Conference on Machine Learning

\bibitem[{Li et~al(2017)Li, Ma, Jiang, Li, Zhang, Liu, Cao, Kannan, and
  Zhu}]{li2017deep}
Li C, Ma X, Jiang B, Li X, Zhang X, Liu X, Cao Y, Kannan A, Zhu Z (2017) Deep
  speaker: an end-to-end neural speaker embedding system. arXiv preprint
  arXiv:170502304

\bibitem[{McLaughlin et~al(2016)McLaughlin, Martinez~del Rincon, and
  Miller}]{mclaughlin2016recurrent}
McLaughlin N, Martinez~del Rincon J, Miller P (2016) Recurrent convolutional
  network for video-based person re-identification. In: Proceedings of the IEEE
  conference on computer vision and pattern recognition, pp 1325--1334

\bibitem[{Mikolov et~al(2013)Mikolov, Sutskever, Chen, Corrado, and
  Dean}]{mikolov2013distributed}
Mikolov T, Sutskever I, Chen K, Corrado GS, Dean J (2013) Distributed
  representations of words and phrases and their compositionality. In: Advances
  in neural information processing systems, pp 3111--3119

\bibitem[{Mital et~al(2011)Mital, Smith, Hill, and
  Henderson}]{mital2011clustering}
Mital PK, Smith TJ, Hill RL, Henderson JM (2011) Clustering of gaze during
  dynamic scene viewing is predicted by motion. Cognitive Computation
  3(1):5--24

\bibitem[{Mueller and Thyagarajan(2016)}]{mueller2016siamese}
Mueller J, Thyagarajan A (2016) Siamese recurrent architectures for learning
  sentence similarity. In: Thirtieth AAAI Conference on Artificial Intelligence

\bibitem[{Neculoiu et~al(2016)Neculoiu, Versteegh, and
  Rotaru}]{neculoiu2016learning}
Neculoiu P, Versteegh M, Rotaru M (2016) Learning text similarity with siamese
  recurrent networks. In: Proceedings of the 1st Workshop on Representation
  Learning for NLP, pp 148--157

\bibitem[{Oh~Song et~al(2016)Oh~Song, Xiang, Jegelka, and
  Savarese}]{oh2016deep}
Oh~Song H, Xiang Y, Jegelka S, Savarese S (2016) Deep metric learning via
  lifted structured feature embedding. In: Proceedings of the IEEE Conference
  on Computer Vision and Pattern Recognition, pp 4004--4012

\bibitem[{Reed et~al(2016)Reed, Akata, Lee, and Schiele}]{reed2016learning}
Reed S, Akata Z, Lee H, Schiele B (2016) Learning deep representations of
  fine-grained visual descriptions. In: Proceedings of the IEEE Conference on
  Computer Vision and Pattern Recognition, pp 49--58

\bibitem[{Resnick(2013)}]{resnick2013extreme}
Resnick SI (2013) Extreme values, regular variation and point processes.
  Springer

\bibitem[{Schroff et~al(2015)Schroff, Kalenichenko, and
  Philbin}]{schroff2015facenet}
Schroff F, Kalenichenko D, Philbin J (2015) Facenet: A unified embedding for
  face recognition and clustering. In: Proceedings of the IEEE conference on
  computer vision and pattern recognition, pp 815--823

\bibitem[{Sedighi and Fridrich(2017)}]{sedighi2017histogram}
Sedighi V, Fridrich J (2017) Histogram layer, moving convolutional neural
  networks towards feature-based steganalysis. Electronic Imaging
  2017(7):50--55

\bibitem[{Sohn(2016)}]{sohn2016improved}
Sohn K (2016) Improved deep metric learning with multi-class n-pair loss
  objective. In: Advances in Neural Information Processing Systems, pp
  1857--1865

\bibitem[{Vilnis and McCallum(2015)}]{vilnis2014word}
Vilnis L, McCallum A (2015) Word representations via {Gaussian} embedding.
  International Conference on Learning Representations (ICLR)

\bibitem[{Wang et~al(2017)Wang, Zhou, Wen, Liu, and Lin}]{wang2017deep}
Wang J, Zhou F, Wen S, Liu X, Lin Y (2017) Deep metric learning with angular
  loss. In: Proceedings of the IEEE International Conference on Computer
  Vision, pp 2593--2601

\bibitem[{Wang et~al(2016)Wang, Li, Ouyang, and Wang}]{wang2016learnable}
Wang Z, Li H, Ouyang W, Wang X (2016) Learnable histogram: Statistical context
  features for deep neural networks. In: European Conference on Computer
  Vision, Springer, pp 246--262

\bibitem[{Wu et~al(2017)Wu, Manmatha, Smola, and Krahenbuhl}]{wu2017sampling}
Wu CY, Manmatha R, Smola AJ, Krahenbuhl P (2017) Sampling matters in deep
  embedding learning. In: Proceedings of the IEEE International Conference on
  Computer Vision, pp 2840--2848

\bibitem[{Wu et~al(2018)Wu, Wang, Gao, and Li}]{wu2018and}
Wu L, Wang Y, Gao J, Li X (2018) Where-and-when to look: Deep siamese attention
  networks for video-based person re-identification. IEEE Transactions on
  Multimedia 21(6):1412--1424

\bibitem[{Yuan et~al(2017)Yuan, Yang, and Zhang}]{yuan2017hard}
Yuan Y, Yang K, Zhang C (2017) Hard-aware deeply cascaded embedding. In:
  Proceedings of the IEEE international conference on computer vision, pp
  814--823

\bibitem[{Zhang et~al(1995)Zhang, Begleiter, Porjesz, Wang, and
  Litke}]{zhang1995event}
Zhang XL, Begleiter H, Porjesz B, Wang W, Litke A (1995) Event related
  potentials during object recognition tasks. Brain Research Bulletin
  38(6):531--538

\end{thebibliography}
\bibliographystyle{spbasic}

\end{document}